\documentclass{article}

\usepackage{hyperref}

\usepackage[accepted]{icml2020}

\usepackage{savesym,multirow,url,xspace,graphicx,hyperref,enumitem,color}
\usepackage[titletoc]{appendix}
\usepackage{subcaption}
\usepackage{dsfont}
\usepackage{multicol}

\usepackage{etoolbox}

\makeatletter
\newif\if@restonecol
\makeatother

\usepackage{amsmath,amssymb,xfrac,amsthm}
\setitemize{noitemsep,topsep=1pt,parsep=1pt,partopsep=1pt, leftmargin=12pt}
\newtheorem{lemma}{Lemma}
\newtheorem{theorem}{Theorem}
\newtheorem*{theorem*}{Theorem}

\newtheorem{corollary}{Corollary}

\newtheorem{definition}{Definition}

\makeatletter

\newcommand{\Rmnum}[1]{\expandafter\@slowromancap\romannumeral #1@}
\makeatother
\usepackage{caption}

\usepackage{enumitem}
\usepackage{wasysym}

\DeclareMathOperator*{\argmin}{arg\,min}
\DeclareMathOperator*{\argmax}{arg\,max}

\newcommand{\cS}{{S}}

\newcommand{\cA}{\mathcal{A}}

\newcommand{\R}{\mathbb{R}}

\newcommand{\bigObound}{\ensuremath{\mathcal{O}}}
\newcommand{\smallObound}{\ensuremath{o}}

\newcommand{\epsP}{\ensuremath{\delta}}
\newcommand{\NoAttack}{\textsc{None}}
\newcommand{\NTRAttack}{\textsc{NT-RAttack}}
\newcommand{\NTDAttack}{\textsc{NT-DAttack}}
\newcommand{\RAttack}{\textsc{RAttack}}
\newcommand{\DAttack}{\textsc{DAttack}}

\newcommand{\regret}{\textsc{Regret}}

\newcommand{\avgmissm}{\textsc{AvgMiss}}
\newcommand{\avgcost}{\textsc{AvgCost}}

\newcommand{\targetpi}{\ensuremath{\pi_{\dagger}}}
\newcommand{\neighbor}[3]{#1\{#2;#3\}}

\usepackage{bm}
\usepackage{bbm}

\newcommand{\norm}[1]{\left\lVert#1\right\rVert}

\newcommand{\expct}[1]{\mathbb{E}\left[#1\right]}
\newcommand{\expctu}[2]{\mathbb{E}_{#1}\left[#2\right]}

\newcommand{\Pru}[2]{\ensuremath{\mathbb{P}_{#1}\left[#2\right] }}
\newcommand{\ind}[1]{\mathds{1}\left[#1\right]}
\newcommand{\pos}[1]{\left[#1\right]^+}

\icmltitlerunning{Policy Teaching via Environment Poisoning}
 
\begin{document}

\twocolumn[
\icmltitle{Policy Teaching via Environment Poisoning:\\Training-time Adversarial Attacks against Reinforcement Learning}



\begin{icmlauthorlist}
  \icmlauthor{Amin Rakhsha}{af1}
  \icmlauthor{Goran Radanovic}{af1}
  \icmlauthor{Rati Devidze}{af1}
  \icmlauthor{Xiaojin Zhu}{af2}
  \icmlauthor{Adish Singla}{af1}
\end{icmlauthorlist}
\icmlaffiliation{af1}{Max Planck Institute for Software Systems (MPI-SWS).}
\icmlaffiliation{af2}{University of Wisconsin-Madison}
\icmlcorrespondingauthor{Adish Singla}{adishs@mpi-sws.org}
\icmlkeywords{Machine Learning, ICML}
\vskip 0.3in
]



\printAffiliationsAndNotice{}  

\newtoggle{longversion}
\settoggle{longversion}{true}
\begin{abstract}
We study a security threat to reinforcement learning where an attacker poisons the learning environment to force the agent into executing a target policy chosen by the attacker. As a victim, we consider RL agents whose objective is to find a policy that maximizes average reward in undiscounted infinite-horizon problem settings. The attacker can manipulate the rewards or the transition dynamics in the learning environment at training-time and is interested in doing so in a stealthy manner. We propose an optimization framework for finding an optimal stealthy attack for different measures of attack cost. We provide sufficient technical conditions under which the attack is feasible and provide lower/upper bounds on the attack cost. We instantiate our attacks in two settings: (i) an offline setting where the agent is doing planning in the poisoned environment, and (ii) an online setting where the agent is learning a policy using a regret-minimization framework with poisoned feedback. Our results show that the attacker can easily succeed in teaching any target policy to the victim under mild conditions and highlight a significant security threat to reinforcement learning agents in practice.
\end{abstract}

\section{Introduction}\label{sec.introduction}
Understanding adversarial attacks on learning algorithms is critical to finding security threats against the deployed machine learning systems and in designing novel algorithms robust to  those threats. We focus on \emph{training-time} adversarial attacks on learning algorithms, also known as data poisoning attacks \cite{huang2011adversarial,biggio2018wild,zhu2018optimal}. Different from \emph{test-time} attacks where the adversary perturbs test data to change the algorithm's decisions,  poisoning attacks manipulate the training data to change the algorithm's decision-making policy itself.

Most of the existing work on data poisoning attacks has focused on supervised learning algorithms~\cite{DBLP:conf/icml/BiggioNL12,mei2015using,DBLP:conf/icml/XiaoBBFER15,alfeld2016data,li2016data,DBLP:journals/corr/abs-1811-00741}. 
In contemporary works, researchers have explored data poisoning attacks against stochastic multi-armed bandits \cite{DBLP:conf/nips/Jun0MZ18,DBLP:conf/icml/LiuS19a} and contextual bandits \cite{DBLP:conf/gamesec/MaJ0018}, which belong to family of online learning algorithms with limited feedback--- such algorithms are widely used in real-world applications such as news article recommendation~\cite{DBLP:conf/www/LiCLS10} and web advertisements ranking~\cite{DBLP:journals/tist/ChapelleMR14}. The feedback loop in online learning makes these applications easily susceptible to data poisoning, e.g., attacks in the form of  click baits \cite{miller2011s}. In this paper, we focus on data poisoning attacks against reinforcement learning (RL) algorithms, an online learning paradigm for sequential decision-making under uncertainty~\cite{sutton2018reinforcement}.\footnote{Poisoning attacks is also mathematically equivalent to the formulation of machine teaching with teacher being the adversary~\cite{zhu2018overview}. However, the problem of designing optimized environments for teaching a target policy to an RL agent is not well-understood in machine teaching.}
 Given that RL algorithms are increasingly used in critical applications, including cyber-physical systems~\cite{li2019reinforcement} and personal assistive devices \cite{rybski2007interactive}, it is of utmost importance to investigate the security threat to RL algorithms against different forms of poisoning attacks.


%
%
%

\subsection{Overview of our Results and Contributions}\label{sec.introduction.contributions}
In the following, we discuss a few of the types/dimensions of poisoning attacks in RL in order to  highlight the novelty of our work in comparison to existing work.

\textbf{Type of adversarial manipulation.} Existing works on poisoning attacks against RL have studied the manipulation of rewards only  \cite{DBLP:conf/aaai/ZhangP08,DBLP:conf/sigecom/ZhangPC09,ma2019policy,DBLP:conf/gamesec/HuangZ19a}. A key novelty of our work is that we study environment poisoning, i.e., manipulating rewards or manipulating transition dynamics.
For certain applications, it is more natural to manipulate the transition dynamics instead of the rewards, such as (i) the inventory management problem setting where state transitions are controlled by demand and supply of products in a market and (ii) conversational agents where the state is represented by the history of conversations.
%
Through our proposed optimization framework, we show that the problem of optimally poisoning transition dynamics is a lot more challenging than that of poisoning rewards, and the attack might not be always feasible. We provide sufficient technical conditions which ensures attacker's success and provide lower/upper bounds on the attack cost.

\textbf{Objective of the learning agent.} Existing works have focused on studying RL agents that maximizes cumulative reward in \emph{discounted} problem settings. In our work, we consider RL agents that maximizes average reward in \emph{undiscounted} infinite-horizon settings~\cite{Puterman1994,DBLP:journals/ml/Mahadevan96}---this is a more suitable objective for many real-world applications that have cyclic tasks or tasks without absorbing states, e.g., inventory management and scheduling problems \cite{tadepalli1994h,Puterman1994}, and a robot learning to avoid obstacles \cite{DBLP:journals/ai/MahadevanC92}. This subtle difference in RL objective leads to  technical challenges because of the absence of the contraction property that comes from the usual discount factor \cite{DBLP:journals/ml/Mahadevan96}. The results and bounds we provide are based on several measures of the problem setting including \emph{Hajnal measure} of the transition kernel and \emph{diameter of the Markov chain} induced by target policy in the original unpoisoned environment.
Another reason we are considering average reward objective is because there are well-studied online RL algorithms for this objective using regret-minimization framework as discussed next.

\textbf{Offline planning and online learning.} Existing works have focused on attacks in an \emph{offline} setting where the adversary first poisons the reward function in the environment and then the RL agent finds a policy via \emph{planning} \cite{DBLP:conf/aaai/ZhangP08,DBLP:conf/sigecom/ZhangPC09,ma2019policy}. In contrast, we call a setting as \emph{online} where the adversary interacts with a \emph{learning} agent to manipulate the feedback signals. One of the key differences in these two settings is in measuring attacker's cost: The $\ell_\infty$\emph{-norm} of manipulation is commonly studied for the offline setting; for the online setting, the cumulative cost of attack over time (e.g., measured by $\ell_1$\emph{-norm} of manipulations) is more relevant and has not been studied in literature.
We instantiate our attacks in both the settings with appropriate notions of attack cost. For the online learning setting, we consider an agent who is learning a policy using regret-minimization framework \cite{auer2007logarithmic,jaksch2010near}.
We note that our attacks are constructive, and we provide numerical simulations to support our theoretical statements. Our results demonstrate that the attacker can easily succeed in teaching (forcing) the victim to execute the desired target policy at a minimal cost.
\begin{figure*}[t!]
\centering
	\begin{subfigure}[b]{0.48\textwidth}
	   \centering
		\includegraphics[width=1\linewidth]{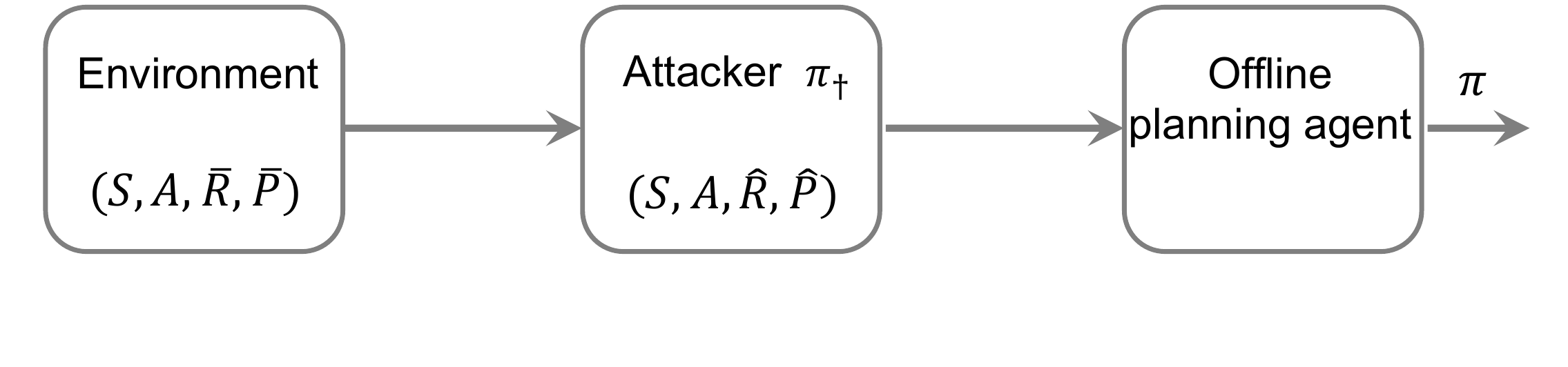}
		\caption{Poisoning attack against an RL agent doing offline planning}
		\label{fig:model.offline}
	\end{subfigure}
	\quad
	\begin{subfigure}[b]{0.48\textwidth}
	    \centering
		\includegraphics[width=1\linewidth]{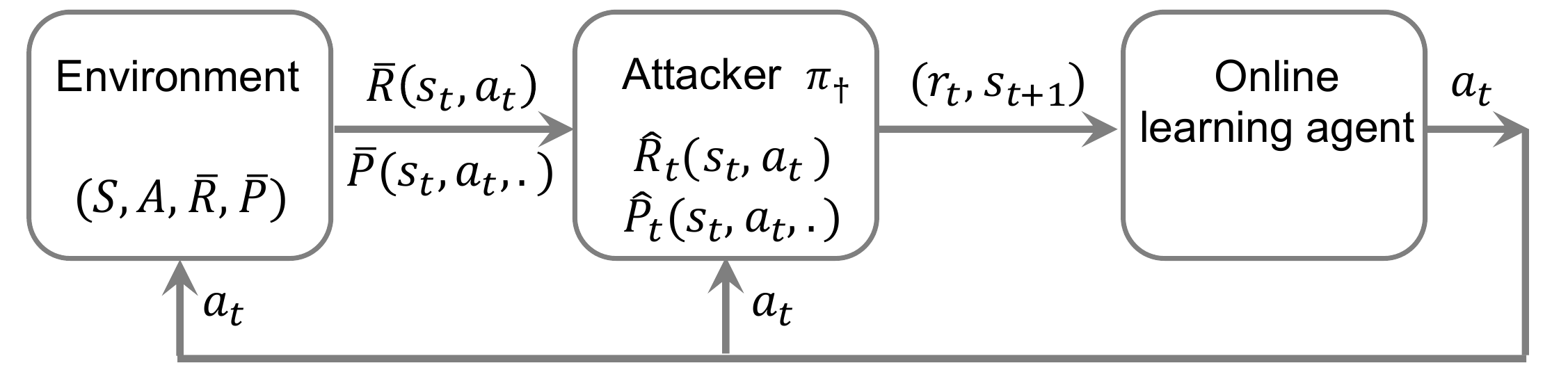}
		\caption{Poisoning attack against an RL agent doing online learning}
		\label{fig:model.online}
	\end{subfigure}
   \caption{(a) Adversary first poisons the environment by manipulating reward function and transition dynamics, then, the RL agent finds an optimal policy via \emph{planning} algorithms based on Dynamic Programming \cite{Puterman1994}. (b) Adversary interacts with an RL agent to manipulate the feedback signals; here, we consider an agent who is learning a policy using regret-minimization framework \cite{auer2007logarithmic,jaksch2010near}.}
	\label{fig:model}
\end{figure*}

\section{Environment and RL Agent}\label{sec.preliminaries}
We consider a standard RL setting, based on Markov decision processes and RL agents that optimize their expected utility. In contrast to prior work on reward poisoning attacks that assume RL agents are optimizing their total discounted rewards, we focus on the average reward optimality criterion, as specified in more detail in the following subsections. 

\subsection{Environment, Policy, and Optimality Criteria}\label{sec.preliminaries.env}
The environment is a Markov Decision Process (MDP) defined as $M = (S, A, R, P)$, where $S$ is the state space, $A$ is the action space, $R\colon S\times A \to \R$ is the reward function, and $P\colon S\times A\times S \to [0, 1]$ is the state transition dynamics, i.e., $P(s, a, s')$ denotes the probability of reaching state $s'$ when taking action $a$ in state $s$.
We denote a deterministic policy by $\pi$, which is a map from states to actions, i.e., $\pi\colon \cS\to\cA$.

\looseness-1
In our work, we consider \emph{average reward} optimality criteria in \emph{undiscounted} infinite-horizon settings~\cite{Puterman1994,DBLP:journals/ml/Mahadevan96}----this optimality criteria is particularly suitable for applications that have cyclic tasks or tasks without absorbing states (see discussion in Section~\ref{sec.introduction.contributions}). Given an initial state $s$, the expected average reward of a policy $\pi$ in MDP $M$ is given by
	$
	\rho(\pi, M, s) := \lim_{N\to \infty} \frac{1}{N} \expct{\sum_{t=0}^{N-1} R(s_t, a_t)|s_0=s, \pi}
	$,
where the expectation is over the rewards received by the agent when starting from initial state $s_0 = s$ and following the policy $\pi$.

Under mild conditions, the above limit exists and the value $\rho(\pi, M, s)$ is independent of the initial state $s$. In this paper, we will assume the condition that the MDP is \textit{ergodic}, which implies that every policy $\pi$ has a stationary distribution $\mu^\pi$, and $\mu^\pi(s) > 0$ for every state $s$ \cite{Puterman1994}. When ergodicity holds, the average reward or \emph{gain} of a policy can be computed as $\rho(\pi, M) = \sum_{s \in S}\mu^\pi(s)R(s, \pi(s))$; we also denote this as $\rho^\pi$ when the MDP $M$ is clear from context. For an overview of average-reward RL, we refer the reader to \cite{Puterman1994,DBLP:journals/ml/Mahadevan96,even2005experts}.



A policy $\pi^*$ is optimal if for every other deterministic policy $\pi$ we have $\rho^{\pi^*} \ge \rho^{\pi},$ and $\epsilon$-robust optimal if $\rho^{\pi^*} \ge \rho^{\pi} + \epsilon$ also holds.
Finally, we define a coefficient
$
\alpha = \min_{s,a,s',a'} \sum_{x \in S}\min \big (P(s,a,x), P(s',a',x) \big )
$.
Note that $(1-\alpha)$ is equivalent to Hajnal measure of $P$~\cite{Puterman1994}.

\subsection{RL Agent}\label{subsec.online_agent}
We consider RL agents with average reward optimality criteria in the following two settings (also, see Figure~\ref{fig:model}).

\textbf{Offline planning agent.}
In the \emph{offline} setting, an RL agent is given an MDP $M$, and chooses a deterministic optimal policy $\pi^* \in \argmax_{\pi} \rho(\pi, M)$. 
The optimal policy can be found via \emph{planning} algorithms based on  Dynamic Programming such as value iteration~\cite{Puterman1994}. 

\textbf{Online learning agent}.
In the \emph{online} setting, an RL agent does not know the MDP $M$ (i.e., $R$ and $P$ are unknown). At each step $t$, the agent stochastically chooses an action $a_t$ based on the previous observations, and then as feedback it obtains reward $r_t$ and transitions to the next state $s_{t+1}$. 
In this paper, we consider an agent who is using a learning algorithm to take actions based on regret-minimization framework. Performance of such a learner in MDP $M$ is measured by its \emph{regret} which after $T$ steps is given by
 $\regret(T, M) = \rho^* \cdot T - \sum_{t=0}^{T-1} r_t$
 ,
where $\rho^* := \rho(\pi^*, M)$ is the optimal average reward. We only consider agents with an algorithm promising a sublinear bound for total expected regret, i.e., $\expct{\regret(T, M)}$ is $\smallObound(T)$.~We note that well-studied algorithms with sublinear regret exist for average reward criteria, e.g., UCRL algorithm~\cite{auer2007logarithmic,jaksch2010near} and algorithms based on posterior sampling method~\cite{agrawal2017optimistic}.  



\section{Attack Models and Problem Formulation}\label{sec.formulation}
In this section, we formulate the problem of adversarial attacks on the RL agent in both the offline and online settings. In what follows, the original MDP (before poisoning) is denoted by $\overline{M} = (S, A, \overline{R}, \overline{P})$, and an \emph{overline} is added to the corresponding quantities before poisoning, such as $\overline{\rho}^\pi$, $\overline{\mu}^\pi$, and $\overline{\alpha}$.
%
%

The attacker has a target policy $\targetpi$ and poisons the environment with the \emph{goal} of teaching/forcing the RL agent to executing this policy.\footnote{Our results can be translated to attacks against the Batch RL agent studied by \cite{ma2019policy} where the attacker poisons the training data used by the agent to learn MDP parameters.}
The attacker is interested in doing a stealthy attack with minimal \emph{cost} to avoid being detected.\footnote{Note that the attacks that we will study in subsequent sections will be poisoning either rewards only or transition dynamics only. In a general attack manipulating both the rewards and dynamics simultaneously, one needs to appropriately combine the costs (also, refer to discussion in Section~\ref{sec.conclusions}). \label{footnote.note-on-cost-of-attack}}
We assume that the attacker knows the original MDP $\overline{M}$, i.e., the original reward function $\overline{R}$ and state transition dynamics $\overline{P}$. This assumption is standard in the existing literature on poisoning attacks against RL. 
The attacker requires that the RL agent behaves as specified in Section~\ref{sec.preliminaries}, however the attacker doesn't know the agent's algorithm or internal parameters.
\subsection{Attack Against an Offline Planning Agent}\label{sec.formulation.offline}
In attacks against an offline planning agent, the attacker manipulates the original MDP $\overline{M} = (S, A, \overline{R}, \overline{P})$ to a poisoned MDP  $\widehat{M} = (S, A, \widehat{R}, \widehat{P})$ which is then used by the RL agent for finding the optimal policy, see Figure~\ref{fig:model.offline}.

\textbf{Goal of the attack.} Given a margin parameter $\epsilon$, the attacker poisons the reward function or the transition dynamics so that the target policy $\targetpi$ is $\epsilon$-robust optimal in the poisoned MDP $\widehat{M}$, i.e., the following condition holds:
\begin{align}\label{prob.off.gen_constraints}
    \rho(\targetpi, \widehat{M}) \ge \rho(\pi, \widehat{M}) + \epsilon, \quad \forall \pi\ne\pi^\dagger.
\end{align}
%
%
\textbf{Cost of the attack.} We consider two notions of costs defined in terms of $\ell_p$-norm of differences in reward functions (i.e., $\overline{R}$ and $\widehat{R}$) and in transition dynamics (i.e., $\overline{P}$ and $\widehat{P}$).
For attacks via poisoning rewards, we  quantify the cost as
\begin{align*}
\norm{\widehat{R} - \overline{R}}_{p} = \bigg (\sum_{s,a} \Big (\big|\widehat{R}(s,a) - \overline{R}(s,a)\big|\Big)^{p}\bigg)^{1/p}
,
\end{align*}
where we treated the reward functions as vectors of length $|S| \cdot |A|$ with values $R(s, a)$. For attacks via poisoning transition dynamics, we  quantify the cost as follows:
\begin{align*}
     \norm{\widehat{P} - \overline{P}}_{p} = \bigg (\sum_{s,a} \Big(\sum_{s'}\big|\widehat{P}(s,a,s') - \overline{P}(s,a,s')\big|\Big)^{p}\bigg)^{1/p}.
     %
\end{align*}
Here, we have computed $\ell_p$-norm of a vector of length $|S| \cdot |A|$ where the value for each $(s, a)$ is given by the $\ell_1$-\emph{norm} of difference of two distributions $\widehat{P}(s,a,.)$ and $\overline{P}(s,a,.)$.
\subsection{Attack Against an Online Learning Agent}\label{sec.formulation.online}
In attacks against an online learning agent, the attacker at time $t$ manipulates the reward function $\overline{R}(s_t, a_t)$ and transition dynamics $\overline{P}(s_t, a_t, .)$ for the current state $s_t$ and agent's action $a_t$, see Figure~\ref{fig:model.online}. Then, at time $t$,  the (poisoned) reward $r_t$ is obtained from $\widehat{R}_t (s_t,a_t)$ instead of $\overline{R}(s_t,a_t)$ and the (poisoned) next state $s_{t+1}$ is sampled from $\widehat{P}_t(s_t,a_t,.)$ instead of $\overline{P}(s_t,a_t,.)$.



\textbf{Goal of the attack.} Specification of the attacker's goal in this online setting is not as straightforward as that in the offline setting, primarily because the agent might never converge to any stationary policy. In our work, at time $t$ when the current state is $s_t$, we measure the mismatch of agent's action $a_t$ w.r.t. the target policy $\targetpi$ as $\ind{a_t \ne \targetpi(s_t)}$ where $\ind{.}$ denotes the indicator function. With this, we define a notion of \emph{average mismatch} of learner's actions in time horizon $T$ as follows:
\begin{align}
    \avgmissm(T) = \frac{1}{T} \cdot \bigg(\sum_{t=0}^{T-1} \ind{a_t \ne \targetpi(s_t)}\bigg).
\end{align}
The goal of the attacker is to  ensure that $\avgmissm(T)$ is $o(1)$, or, alternatively, the total number of time steps where there is a mismatch is $o(T)$.

\textbf{Cost of the attack.} We consider a notion of \emph{average cost} of attack in time horizon $T$ denoted as $\avgcost(T)$.
For reward poisoning attacks, 
$\avgcost(T)$ is
\begin{align*}
    \frac{1}{T} \cdot \bigg(\sum_{t=0}^{T-1} \Big(\big|\widehat{R}_t (s_t,a_t) - \overline{R}(s_t, a_t)\big|\Big)^p\bigg)^{1/p}, 
\end{align*}
%
and for dynamics poisoning attacks, 
$\avgcost(T)$ is
\begin{align*}
    \frac{1}{T} \cdot \bigg(\sum_{t=0}^{T-1} \Big(\sum_{s'} \big|\widehat{P}_t(s_t,a_t,s') - \overline{P}(s_t,a_t,s')\big| \Big)^p\bigg)^{1/p}. 
\end{align*}
Note that here the $\ell_p$-norm is defined over a vector of length $T$ with values quantifying the attack cost at each time step $t$. One of the key differences in measuring attacker's cost for offline and online settings is the use of appropriate norm. While the $\ell_\infty$\emph{-norm} of manipulation is most suitable and commonly studied for the offline setting; for the online setting, the cumulative cost of attack over time measured by $\ell_1$\emph{-norm} is most relevant.
\section{Attacks in Offline Setting}\label{sec.offlineattacks}
In this section, we introduce and analyze attacks against an offline planing agent that derives its policy using a poisoned MDP $\widehat M$.  The attacker tries to minimally change the original MDP $\overline{M}$, while at the same time ensuring that the target policy is optimal in the modified MDP $\widehat M$. 


\subsection{Overview of the Approach}\label{sec.off.ideas-definitions}
To formulate optimization problems explicitly and provide bounds on the quality of obtainable solutions, we first define concepts that enable us do the following: i) reduce the complexity of achieving the attacker's goal explained in Section \ref{sec.formulation.offline}, ii) quantify how much change is required in rewards or transitions. 
To find MDP $\widehat M$ for which the target policy is $\epsilon$-robust optimal, one could directly utilize constraints expressed by \eqref{prob.off.gen_constraints}. However, the number of constraints in \eqref{prob.off.gen_constraints} equals $(|A|^{|S|}-1)$, i.e., it is exponential in $|S|$, making optimization problems that directly utilize them intractable. We show that it is enough to satisfy these constraints for $(|S|\cdot|A|-|S|)$ policies that we call \textit{neighbors} of the target policy and which we define as follows:
\begin{definition}
	For a policy $\pi$, its \textit{neighbor} policy $\neighbor{\pi}{s}{a}$ is defined as
	$$
	\neighbor{\pi}{s}{a}(x) = \left\{\begin{array}{lc}
		\pi(x) & x \ne s \\
		a & x = s
	\end{array}\right. .
	$$
\end{definition}
The following lemma provides a simple verification criteria for examining whether a policy of interest is $\epsilon$-robust optimal in a given MDP. Its proof can be found in the supplementary  material.    
\begin{lemma}
	\label{lemma.using_neighbors}
	Policy $\pi$ is $\epsilon$-robust optimal \emph{iff} we have $\rho^{\pi} \ge \rho^{\neighbor{\pi}{s}{a}} + \epsilon$ for every state $s$ and action $a \ne \pi(s)$.
\end{lemma}
In other words, Lemma \ref{lemma.using_neighbors} implies that (sub)optimality of the target policy can be deduced by examining its neighbor policies.
By using the definition of average rewards, we conclude that the attacker's strategy, which consists of modifying MDP $\overline{M} = (S, A, \overline{R}, \overline{P})$ to MDP $\widehat M = (S, A, \widehat{R}, \widehat{P})$, is successful if and only if for each state $s$ and action $a \ne \targetpi(s)$ we have:
\begin{align}\label{off.attack.nec.cond}
    \sum_{s'} \widehat{\mu}^{\targetpi}(s') &\cdot \widehat R \big (s', \targetpi(s') \big ) \ge \\
    \notag
	&\sum_{s'} \widehat{\mu}^{\neighbor{\targetpi}{s}{a}}(s') \cdot \widehat R \big (s', \neighbor{\targetpi}{s}{a}(s') \big ) + \epsilon.
\end{align}

 To simplify the exposition of our results, we introduce the following quantity w.r.t. MDP $\overline M$ for all $s \in S, a \in A$:
\begin{align}
    \label{chi.definition}
    \overline \chi^{\pi}_{\epsilon}(s, a) = 
    \begin{cases}
        \pos{\frac{
                \overline \rho^{\neighbor{\pi}{s}{a}} - \overline \rho^{\pi} + \epsilon
            }{
                \overline \mu^{\neighbor{\pi}{s}{a}}(s)
            }} \quad & \textnormal{ for } a\ne\pi(s), \\
            0 &  \textnormal{ for }  a = \pi(s).
    \end{cases}
\end{align}
Here, $\pos{x}$ is equal to $\max\{0, x\}$. As we show in our formal results, the amount of change needed in modifying the original MDP $\overline{M}$ to MDP $\widehat{M}$ so that \eqref{off.attack.nec.cond} holds is captured through $\overline \chi^{\pi}_{\epsilon}$.


\subsection{Attacks in Offline Setting via Poisoning Rewards}
\label{sec.off.reward}

The first attack we consider is an attack on an offline planning agent via poisoning rewards, where the attacker's strategy consists of choosing a new reward function which makes the target policy $\epsilon$-robust optimal. In this attack we have $\widehat{P} = \overline{P}$, and hence $\widehat{\mu}^\pi$ is equal to $\overline{\mu}^\pi$ for any policy $\pi$, so $\widehat{\mu}^\pi$ can be precomputed based on MDP $\overline{M}$. This allows us to formulate an optimization problem for finding an optimal reward poisoning attack using condition \eqref{off.attack.nec.cond}:
\begin{align}
	\label{prob.off.unc}
	\tag{P1}
	&\quad \min_{R} \quad \norm{R - \overline R}_{p}\\
	\notag
	&\quad \mbox{ s.t. } \quad \text{condition \eqref{off.attack.nec.cond} holds with  $\widehat R = R$ and  $\widehat \mu = \overline{\mu}$}.
\end{align}
In this optimization problem, the constraints are obtained by using condition \eqref{off.attack.nec.cond} for specific instantiation of $\widehat{R}$ and $\widehat{\mu}$. In particular, we set 
$\widehat{R}$ to be the optimization variable $R$, and $\widehat{\mu}$ to be the stationary distributions of the original transition kernel $\overline{P}$ (i.e., $\overline{\mu}$).

The optimization problem \eqref{prob.off.unc} is a tractable convex program 
with
linear constraints, and  we show it is always feasible. 
Additionally, the following results provide lower and upper bounds on the attack cost, that is, on the cost of solution $\widehat R$. 
\begin{theorem}
\label{theorem_off_unc}
	The optimization problem \eqref{prob.off.unc} is always feasible, and the cost of the optimum solution is bounded by
	\begin{align*}
	   \frac{\overline \alpha}{2}\cdot\norm{\overline \chi^{\targetpi}_{\epsilon}}_{\infty} \le \norm{\widehat R-\overline R}_{p} \le \norm{\overline \chi^{\targetpi}_{\epsilon}}_{p}. 
	\end{align*}
\end{theorem}
The proof of the theorem can be found in the supplementary  material. 
The lower bound on the cost of the attack is obtained by extending the proof technique of \cite{ma2019policy} to the average reward criterion, and it depends on two factors. 
The first factor, $\frac{\overline \alpha}{2}$ (see the definition in Section~\ref{sec.preliminaries.env}), is related to the Hajnal measure of the original MDP $\overline{M}$. The second factor, $\norm{\overline \chi^{\targetpi}_{\epsilon}}_{\infty}$, captures the initial disadvantage of the target policy relative to its neighbor policies, as can be seen from the inequality $\overline \chi^{\targetpi}_{\epsilon}(s,a) \ge (\overline \rho^{\neighbor{\targetpi}{s}{a}} - \overline \rho^{\targetpi})$.

The upper bound is obtained by analyzing the optimum solution to a modified version of the optimization problem \eqref{prob.off.unc} which enforces reward function $\widehat R$ to be equal to $\overline R$ for state-action pairs that define the target policy. We further discuss the modified optimization problem in Section~\ref{sec.on.reward}.

\subsection{Attacks in Offline Setting via Poisoning Dynamics}
\label{sec.off.dynamics}
Let us now consider attacks on an offline planning agent based on poisoning transition dynamics. The attacker's strategy now consists of choosing a new transition dynamics $\widehat P$ that makes the target policy $\epsilon$-robust optimal. Again, we would like to utilize condition \eqref{off.attack.nec.cond}, however, now  we cannot precompute $\widehat{\mu}$ since we are modifying the transition dynamics. We have to explicitly account for that in our optimization problem, and we can do so by noting that stationary distribution $\widehat{\mu}^{\pi}$ satisfies:
\begin{align}\label{eq.stat_dist.from.kernel}
    \widehat{\mu}^{\pi}(s) = \sum_{s'} \widehat{P}(s', \pi(s'), s) \cdot \widehat{\mu}^{\pi}(s').
\end{align}
Therefore, by using condition \eqref{off.attack.nec.cond} and $\widehat{R} = \overline{R}$, we can formulate an optimization problem for finding an optimal dynamics poisoning attack:  
\begin{align}
  	\label{prob.off.dyn}
  	\tag{P2}
  	&\quad \min_{P,  \mu^{\targetpi}, \mu^{\neighbor{\targetpi}{s}{a}}} \quad \norm{P - \overline P}_{p}\\
	\notag
	&\quad \mbox{ s.t. } \quad \text{$ \mu^{\targetpi}$ and $P$ satisfy \eqref{eq.stat_dist.from.kernel}}, \\
	\notag
	&\qquad \quad \quad \forall s, a \ne \targetpi(s): \text{$ \mu^{\neighbor{\targetpi}{s}{a}}$ and $P$ satisfy \eqref{eq.stat_dist.from.kernel}}, \\
	\notag
	&\qquad \quad \quad \text{condition \eqref{off.attack.nec.cond} holds with $\widehat R = \overline{R}$ and $\widehat \mu = \mu$},\\
	\notag
	&\qquad \quad \quad \forall s, a, s': P(s,a,s') \ge \epsP \cdot \overline P(s,a,s').
  \end{align}
  Here, $\epsP \in (0, 1]$ in the last set of constraints is a given parameter, specifying how much one is allowed to decrease the original values of transition probabilities. $\epsP > 0$ is a regularity condition which ensures that the new MDP is ergodic.\footnote{This follows because strictly positive trajectory probabilities in MDP $\overline M$ remain strictly positive in the new MDP $\widehat M$, which further implies that all states remain recurrent and  aperiodic. } 
  %
  In the supplementary material, we provide a more detailed discussion on $\epsP$ and how to choose it.   
 Furthermore, the constraints related to condition \eqref{off.attack.nec.cond} are obtained by instantiating $\widehat{R}$ and $\widehat{\mu}$. In particular, $\widehat{R}$ is set to be the original reward function  $\overline{R}$ and $\widehat{\mu}$ is set to be the optimization variables $\mu$ (i.e., $\mu^{\targetpi}$ and $\mu^{\neighbor{\targetpi}{s}{a}}$).
 
%

Since the first and the second set of constraints are quadratic equality constraints, the optimization problem is in general non-convex. Furthermore, the third set of constraints defined by condition \eqref{off.attack.nec.cond}  
might not even be feasible, for example, when reward function is constant or almost constant (i.e., $\overline R(s,a) = \overline R(s',a')$ or more generally $|\overline R(s,a) - \overline R(s',a')| \le \epsilon$). 
In Theorem \ref{theorem_off_dyn} we provide a sufficient condition that renders the optimization problem feasible and we provide bounds on the cost of the attack.
%
Given the nature of the constraints in~\eqref{prob.off.dyn}, we describe an approach for finding an approximate solution in Section \ref{sec.experiments}.

To introduce the formal statement, let us first define relevant quantities. Let $\overline D^{\pi}$ denote the \emph{diameter} of Markov chain induced by policy $\pi$ in MDP $\overline M$, i.e., $\overline D^{\pi} = \max_{s,s'} \overline T^{\pi}(s,s')$, where $\overline T^{\pi}(s,s')$ is the expected time to reach $s'$ starting from $s$ and following policy $\pi$. Next we define {\em value function} of a policy and a few related quantities.
Value function of policy $\pi$ in MDP $\overline M$ is defined as
\begin{align*}
\overline V^\pi(s) = \lim_{N\to \infty} \expct{\sum^{N-1}_{t=0} \Big(\overline{R}(s_t, a_t) - \overline \rho_\pi\Big)| s_0=s, \pi},
\end{align*}
where the expectation is over the rewards received by agent (relative to the offset $\overline \rho_\pi$) when starting from initial state $s_0 = s$ and following policy $\pi$.
Furthermore, let us denote by $\overline B^{\pi}(s) = \sum_{s'} \overline P(s,\pi(s),s') \cdot \overline V^{\pi}(s')$ the expected value of the next state given the current state $s$ and policy $\pi$. We also define $\overline V_{\textnormal{min}}^{\pi} = \min_{s} \overline V^{\pi}(s)$ to be the minimum value of $\overline V^{\pi}$, and $sp(\overline V^{\pi}) = \max_{s} \overline V^{\pi}(s) - \min_{s} \overline V^{\pi}(s)$ to be the span of $\overline V^{\pi}$.
Finally, we introduce the following two quantities important for the theorem statement:
 \begin{itemize}
     \item $\overline \beta^{\pi}_{\epsP}(s, a)$ defined as
     \begin{align}
            \overline \beta^{\pi}_{\epsP}(s,a) &= \overline R(s, \pi(s)) - \overline R(s,a) \\\notag &+
     \overline B^{\pi}(s) - \overline V_{\textnormal{min}}^{\pi} - \epsP \cdot sp(\overline V^\pi).
    \end{align}
    \item $\overline \Lambda(s,a)$ defined as
    \begin{align}\label{eq_def_lambda}
    \overline \Lambda(s,a) = 
    \begin{cases}
        \frac{\overline \chi_0^{\targetpi}( s,a) + \epsilon \cdot (1 + \overline D^{\targetpi})}{ \overline \chi_0^{\targetpi}(s,a) + \overline \beta^{\targetpi}_{\epsP}(s,a)}& \textnormal{if \ } \overline \chi_{\epsilon}^{\targetpi}(s,a) > 0, \\
        0 & \textnormal{otherwise},
    \end{cases}
    \end{align}
    where $\overline \chi_0^{\targetpi}$ is obtained from  $\overline \chi_{\epsilon}^{\targetpi}$ by setting $\epsilon = 0$.  
 \end{itemize}

\begin{theorem}\label{theorem_off_dyn}
If there exists a solution $\widehat P$ to the optimization problem \eqref{prob.off.dyn}, its cost 
    satisfies
    \begin{align*}
       \norm{\widehat P - \overline P}_{p} \cdot  \norm{\overline V^{\targetpi}}_{\infty} \ge \frac{\epsP \cdot \overline \alpha}{2}  \cdot \norm{ \overline \chi_0^{\targetpi}}_{\infty}.
    \end{align*}
    If for every state $s$ and action $a$, it holds that either $\overline \beta^{\targetpi}_{\epsP}(s,a) \ge \epsilon \cdot (1 + \overline D^{\targetpi})$ \texttt{OR} $\overline \chi_{\epsilon}^{\targetpi}(s,a) = 0$, then the optimization problem 
    \eqref{prob.off.dyn}
    has a solution $\widehat P$ whose cost is upper bounded by $\norm{\widehat P - \overline P}_{p} \le 2\cdot \norm{\overline \Lambda}_{p}$.
\end{theorem}
The proof can be found in the supplementary material. The proof technique for the first claim is similar to the one used for proving the lower bound in Theorem \ref{theorem_off_unc}. 

The sufficient condition of 
Theorem \ref{theorem_off_dyn} suggests that the initial disadvantage of the target policy $\targetpi$ relative to its neighbor policy $\neighbor{\targetpi}{s}{a}$ (captured by $\overline \chi_{\epsilon}^{\targetpi}(s,a)$) can be overcome if $\beta^{\targetpi}_{\epsP}(s,a)$ is large enough. 
This in turn means that either $\overline R(s, \targetpi(s))$ is sufficiently larger than $ \overline R(s,a)$, or the future prospect  of the target policy (i.e., $\overline B^{\targetpi}(s)$)  is greater than the myopic disadvantage   (i.e., $\overline R(s, a) - \overline R(s, \targetpi(s))$), by a margin which depends on the value function $\overline V^{\targetpi}$. 
%
The proof for the upper bound is constructive, and it is based on an attack that treats state $s_{\textnormal{sink}} \in \arg\min_{s'} \overline V^{\targetpi}(s')$ as a \emph{sink} state. Then, we shift transitions of state-action pairs that have $\overline \chi_{\epsilon}^{\targetpi}(s,a) > 0$ towards this sink state.
Notice that this is a type of an attack that does not alter the transition dynamics for the the target policy. We further discuss this type of attacks in Section \ref{sec.on.dynamics}.

\section{Attacks in Online Setting}\label{sec.onlineattacks}
We now turn to attacks on an agent that learns over time using the environment feedback. 
Unlike the planning agent from the previous section, an online learning agent derives its policy from the interaction history, i.e., tuples of the form  $(s_{t}, a_{t}, r_{t}, s_{t+1})$. To attack an online learning agent, an attacker changes the environment feedback, i.e.,  reward $r_t$ or state $s_{t+1}$. 


\subsection{Overview of the Approach}
\label{sec.on.ideas-definitions}

 The underlying idea behind our approach is to utilize the fact that a learning agent has a bounded regret, and thus chooses a suboptimal action a bounded number of times. Hence, to steer a learning agent toward selecting the target policy, it suffices for the attacker to provide the feedback (i.e., reward $r_t$ and the next state $s_{t+1}$) sampled from an MDP in which the target policy is $\epsilon$-robust optimal. Such an MDP can be derived using the results from the previous sections.  
We first show that this approach is sound: assuming that an ergodic MDP $M$ has 
$\targetpi$ as its $\epsilon$-robust optimal policy,
the expected number of steps in which a learner whose experience is drawn from MDP $M$ deviates from $\targetpi$ 
is of order $\expct{\regret(T, M)}$.
\begin{lemma}
	\label{lemma.on.known.nontarget}
	Consider an ergodic MDP $M$ that has 
	$\targetpi$ as its $\epsilon$-robust optimal policy,
	and 
	an online learning agent whose expected regret on an MDP $M$ is $\expct{\regret(T, M)}$. The average mismatch of the agent is bounded by $\expct{\avgmissm(T)} \le \frac{1}{T} \cdot K(T, M)$, where
	\begin{align}\label{eq_def_K}
	    K(T, M) = \frac{\mu_{\textnormal{max}}}{\epsilon} \cdot \Big(\expct{\regret(T, M)} + 2\norm{V^{\targetpi}}_\infty\Big),
	\end{align}
with $\mu_{\textnormal{max}} := \max_{s,a}\mu^{\neighbor{\targetpi}{s}{a}} (s)$. Here, $\mu^{\pi}$ and $V^{\pi}$ are respectively the stationary distribution and the value function of a policy $\pi$ on MDP $M$.
\end{lemma}
The proof of the lemma can be found in the supplementary material. The lemma implies that $o(1)$ average mismatch can be achieved in expectation using the sampling based attack described above, 
assuming that $\widehat{M}$ is ergodic and that $\targetpi$ is its $\epsilon$-robust optimal policy. Since the solutions to the optimization problems \eqref{prob.off.unc} and \eqref{prob.off.dyn} from Section \ref{sec.offlineattacks} satisfy this, the lemma applies to them as well. 

However, the expected average cost of such an attack could be $\Omega(1)$ (non-diminishing over time) for a learner with subliner expected regret, unless the original and the sampling MDP have equal rewards and transition probabilities for the state-action pairs of the target policy. Intuitively, if a learner follows the target policy and there exists $s$ for which $\widehat R(s, \targetpi(s)) \ne \overline{R}(s, \targetpi(s)) $  or $\widehat{P}(s, \targetpi(s), .) \ne \overline{P}(s, \targetpi(s), .)$, then the attacker would incur a non-zero cost whenever the learner visits $s$.  
%
To avoid this issue, we need to enforce constraints on the sampling MDP specifying that the attack does not alter rewards and transitions that correspond to the state-action pairs of the target policy. This brings us to the following template that we utilize for attacks on an online learner:
\begin{itemize}
    \item Modify the optimization problems \eqref{prob.off.unc} and \eqref{prob.off.dyn} by respectively adding constraints $\widehat R(s, \targetpi(s)) = \overline{R}(s, \targetpi(s))$ and $\widehat P(s, \targetpi(s), s') = \overline{P}(s, \targetpi(s), s')$. 
    \item Obtain the sampling MDP $\widehat M$ by solving the more constrained version of  \eqref{prob.off.unc} or \eqref{prob.off.dyn}.
    \item Use the sampling MDP $\widehat M$ instead of the environment $\overline M$ during the learning process, i.e., obtain $r_t$ from $\widehat R(s_t, a_t)$ in the case of poisoning rewards and $s_{t+1} \sim \widehat P(s_t, a_t, .)$ in the case of poisoning dynamics (see Figure~\ref{fig:model.online}). 
\end{itemize}

\subsection{Attacks in Online Setting via Poisoning Rewards}
\label{sec.on.reward}



For the case of attacks via poisoning rewards, the sampling MDP $\widehat M$ differs from the original MDP $\overline M$ only in its reward function $\widehat R$. 
Following the attack template from the above, we first find $\widehat R$ using a modified version of  \eqref{prob.off.unc} which ensures that $\widehat R(s, \targetpi(s)) = \overline{R}(s, \targetpi(s))$, i.e.: 
\begin{align}
\label{prob.on.known.reward_n1}
\tag{P3}
\min_{R} &\quad \norm{R - \overline{R}}_{p}\\
\notag
 \mbox{ s.t. }  &\quad  \forall s: \  R(s, \targetpi(s)) = \overline{R}(s, \targetpi(s)),\\
\notag
&\quad \textnormal{all constraints from problem $\eqref{prob.off.unc}$}.
\end{align}
As we show in the supplementary material, the optimal solution to the optimization problem \eqref{prob.on.known.reward_n1} is $\widehat {R}(s,a) = \overline{R}(s,a) - \overline \chi^{\targetpi}_{\epsilon}(s,a)$. Since $\overline \chi^{\targetpi}_{\epsilon}(s,\targetpi(s)) = 0$, using the definition of $\avgcost(T)$ we conclude that an upper bound on $\expct{\avgcost(T)}$ could be computed using (i) the number of times that a non-target action is selected and (ii) the maximum value of $\overline \chi^{\targetpi}_{\epsilon}(s,a)$. The quantity in (i) is specified in the result of Lemma \ref{lemma.on.known.nontarget}, which brings us to the main result of this subsection.
\begin{theorem}
	\label{theorem.on.known.reward}
	Let $\widehat R$ be the optimal solution to \eqref{prob.on.known.reward_n1}.
	 Consider the attack defined by $r_t$ obtained from $\widehat{R}(s_t, a_t)$ and $s_{t + 1} \sim \overline{P}(s_t, a_t, .)$, and an online learning agent whose expected regret on an MDP $M$ is $\expct{\regret(T, M)}$. 
	The average mismatch of the learner is in expectation upper bounded by $$\expct{\avgmissm(T)} \le \frac{K(T, \widehat M)}{T},$$ where $K$ is defined in \eqref{eq_def_K}.
	Furthermore, the average attack cost is in expectation upper bounded by $$\expct{\avgcost(T)} \le \norm{\overline \chi^{\targetpi}_{\epsilon}}_\infty \cdot \frac{\big (K(T, \widehat M)\big )^{1/p}}{T}.$$
\end{theorem}
\subsection{Attacks in Online Setting via Poisoning Dynamics}\label{sec.on.dynamics}

For the case of attacks via dynamics poisoning, the sampling MDP $\widehat M$ differs form the original MDP $\overline M$ in its transition dynamics $\widehat P$. 
Following the attack template,
we can derive the optimization problem for finding $\widehat P$ by simply adding the constraints $P(s, \targetpi(s), s') = \overline{P}(s, \targetpi(s), s')$ in the optimization problem \eqref{prob.off.dyn}, i.e.:
\begin{align}
  	\label{prob.on.known.dyn}
  	\tag{P4}
  	&\quad \min_{P, \mu^{\targetpi},  \mu^{\neighbor{\targetpi}{s}{a}}} \quad \norm{P - \overline P}_{p}\\
	\notag
	&\quad \mbox{ s.t. } \quad \forall s, s': P(s,\targetpi(s),s') =  \overline P(s,\targetpi(s),s'),\\
	\notag
	&\qquad \quad \quad \textnormal{all constraints from problem $\eqref{prob.off.dyn}$}.
 \end{align}
In the supplementary material, we show that the additional constraint allows us to transform \eqref{prob.on.known.dyn} into a tractable convex program with linear constraints.
 In fact, the convex program has a specific structure so that its optimal solution can be obtained by solving $|S|\cdot (|A|-1)$ simple convex problems---each of these simpler problems only involves $|S|$ variables and $|S|+1$ linear constraints.
Moreover, the convex program is feasible if the conditions of Theorem \ref{theorem_off_dyn} are met, and its 
solutions satisfy
$\norm{\widehat P(s,a,.) - \overline P(s,a,.)}_{1} \le 2 \cdot \overline \Lambda(s, a)$, where $\overline \Lambda$ is defined in \eqref{eq_def_lambda}.
%
%
Since $\overline \Lambda(s,\targetpi(s)) = 0$, this means that one can bound $\expct{\avgcost(T)}$ based on the number of times that a non-target action is selected and the maximum value of $\overline \Lambda(s,a)$. 
Combining this insight with Lemma \ref{lemma.on.known.nontarget}, we obtain the following theorem.

%
%
%
\begin{theorem}
	\label{theorem.on.known.dyn}
	%
	Assume that the sufficient condition of Theorem \ref{theorem_off_dyn} holds, 
	and let $\widehat P$ be the optimal solution to \eqref{prob.on.known.dyn}. 
	Consider the attack defined by $r_t$ obtained from $\overline{R}(s_t, a_t)$ and $s_{t + 1} \sim \widehat P(s_t, a_t, .) $, and an online learning agent whose expected regret on an MDP $M$ is $\expct{\regret(T, M)}$.
	The average mismatch of the learner is in expectation upper bounded by
	$$\expct{\avgmissm(T)} \le \frac{ K(T, \widehat M)}{T},$$ where $K$ is defined in \eqref{eq_def_K}.
	Furthermore, the expected attack cost is upper bounded by
	$$\expct{\avgcost(T)} \le 2 \cdot \norm{\overline \Lambda}_\infty \cdot \frac{\big (K(T, \widehat M)\big )^{1/p}}{T},$$ where $\overline \Lambda$ is defined in \eqref{eq_def_lambda}.
\end{theorem}
\looseness-1A direct consequence of Theorem \ref{theorem.on.known.reward} and Theorem \ref{theorem.on.known.dyn} is that for a learner with a sublinear expected regret, the average mismatch and the average cost are expected to decrease over time
 as $\bigObound\Big(\frac{\expct{\regret(T, \widehat{M})}}{T}\Big)$ and $\bigObound\Big(\frac{\expct{\regret(T, \widehat{M})}^{1/p}}{T}\Big)$.
 Note that while we considered $\ell_p$ norms with $p \ge 1$ to define the attack cost, the above results can be generalized to include the case of $p=0$. For $p=0$, the expected value of the average cost is simply upper bounded by  $\frac{K(T, \hat{M})}{T}$, and this follows from Lemma \ref{lemma.on.known.nontarget}.

\section{Numerical Simulations}\label{sec.experiments}
We perform numerical simulations on an environment represented as an MDP with four states and two actions, see Figure~\ref{fig:environment} for details. Even though simple, this environment provides a very rich and an intuitive problem setting to validate the theoretical statements  and understand the effectiveness of the attacks by varying different parameters.
We also vary the number of states in the MDP to check the efficiency of solving different optimization problems, and report run times. We further extend our experimental evaluation in the supplementary material and report results on an additional environment.

\begin{minipage}{\linewidth}
    \makebox[\linewidth]{
    \includegraphics[width=0.6\textwidth]{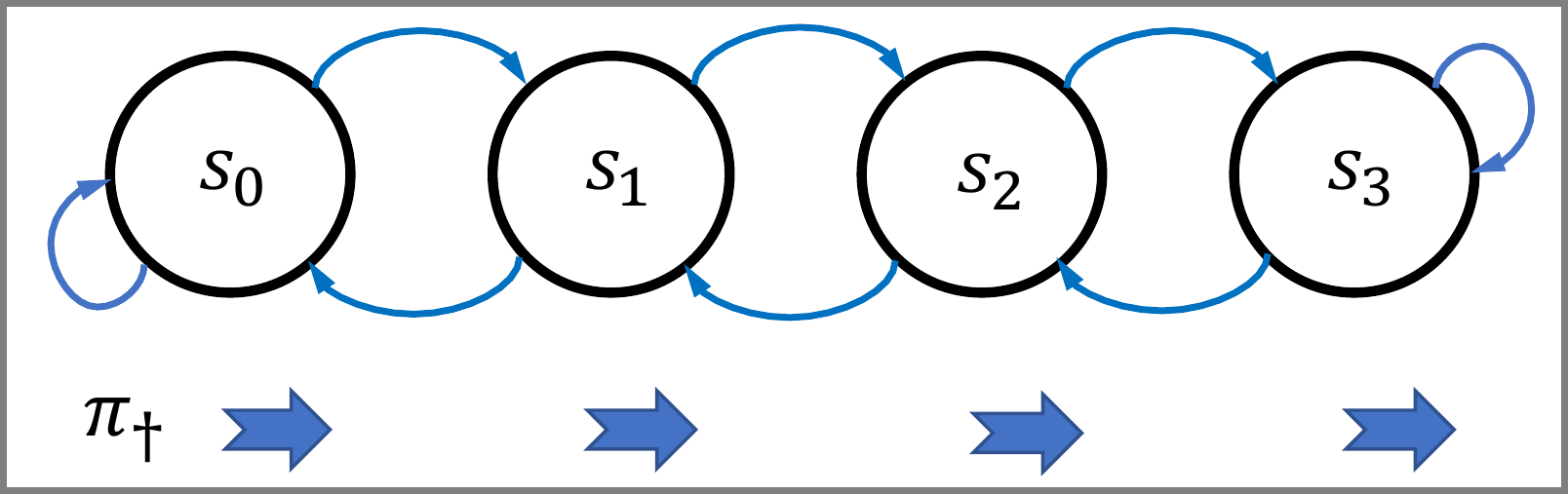}
  	}
  	\captionof{figure}{The environment has $|S| = 4$ states and $|A|=2$ actions given by $\{\texttt{left}, \texttt{right}\}$. The original reward function $\overline{R}$ is action independent and has the following values: $s_1$ and $s_2$ are rewarding states with $\overline{R}(s_1,.) =  \overline{R}(s_2,.) = 0.5$; state $s_3$ has negative reward of $\overline{R}(s_3,.) = -0.5$, and the reward of the state $s_0$ given by $\overline{R}(s_0,.)$  will be varied in experiments. With probability $0.9$, the actions succeed in navigating the agent to left or right as shown on arrows; with probability $0.1$ the agent's next state is sampled randomly from the set $S$. The target policy $\targetpi$ is to take \texttt{right} action in all states as shown in the illustration. 
  	}
	\label{fig:environment}
\end{minipage}
%

\subsection{Attacks in the Offline Setting: Setup and Results}\label{sec.experiments.offline}
\looseness-1For the offline setting, we considered the following attack strategies: (i) \RAttack: reward attacks using $\widehat{R}$ as solution to problem \eqref{prob.off.unc} ($\ell_p$\emph{-norm} with $p=1, 2, \infty$), (ii) \NTRAttack: reward attacks using $\widehat{R}$ as solution to problem \eqref{prob.on.known.reward_n1} (solution is independent of $\ell_p$\emph{-norm}), (iii) \DAttack: dynamics attacks using $\widehat{P}$ as a solution to problem \eqref{prob.off.dyn} ($\ell_p$\emph{-norm} for $p=1, 2, \infty$), and (iv) \NTDAttack: dynamics attacks using $\widehat{P}$ as a solution to problem \eqref{prob.on.known.dyn} (solution is independent of $\ell_p$\emph{-norm}). The regularity parameter $\epsP$ in the problems for solving dynamic poisoning attacks is set to $0.0001$.\footnote{Here, \textsc{NT-} prefix is used to highlight that \emph{non-target only} manipulations are allowed in problems~\eqref{prob.on.known.reward_n1} and \eqref{prob.on.known.dyn}.}

We note that optimal solutions to the problems \eqref{prob.off.unc}, \eqref{prob.on.known.reward_n1}, and \eqref{prob.on.known.dyn} can be computed efficiently using standard optimization techniques (also, refer to discussions in Section~\ref{sec.offlineattacks} and Section~\ref{sec.onlineattacks}). Problem~\eqref{prob.off.dyn} is computationally more challenging, and we provide a simple yet effective approach towards finding an approximate solution by iteratively solving the problem~\eqref{prob.on.known.dyn} as follows: First, we use a simple heuristic to obtain a pool of transition kernels $\widetilde{P}$  by perturbations of $\overline{P}$ that increase the average reward of $\targetpi$, and as second step, we use $\widetilde{P}$'s from this pool as as input to problem \eqref{prob.on.known.dyn} instead of $\overline{P}$. 
So, the runtime of solving problem~\eqref{prob.off.dyn} depends on the number of iterations we invoke problem~\eqref{prob.on.known.dyn} internally. The implementation details and code are provided in supplementary materials.


\looseness-1We vary $\overline{R}(s_0, .) \in [-5, 5]$ and vary $\epsilon$ margin $\in [0, 1]$. We use $\ell_\infty$\emph{-norm} in the measure of attack cost (see Section~\ref{sec.formulation.offline}). The results are reported as an average of $10$ runs (here, \DAttack~is the only stochastic algorithm because of the random initialization as discussed above).


There are two key points we want to highlight in Figure~\ref{fig:results.offline}. First, as we increase $\epsilon$ margin, the attack problem becomes more difficult: While the reward poisoning attacks are always feasible (though with increasing attack cost),  it is infeasible to do dynamics poisoning attacks for $\epsilon > 0.85$. Second, the plots also show that the solution to the generic problems (\RAttack~and \DAttack~for any $\ell_p$\emph{-norm} with $p=1, 2, \infty$) can have much lower cost compared to solutions obtained by problems allowing non-target only manipulations (\NTRAttack~and \NTDAttack).


For the MDP in Figure~\ref{fig:environment} with $|S|=4$, the run times for solving problems \eqref{prob.off.unc}, \eqref{prob.off.dyn}, \eqref{prob.on.known.reward_n1}, and \eqref{prob.on.known.dyn} are roughly $0.0110$s, $3.3010$s, $0.0003$s, and $0.0339$s, respectively. We further ran experiments on a variant of the MDP with $|S|=100$ (keeping the same linear chain structure as in Figure~\ref{fig:environment}), and this led to average run times of $0.8326$s, $157.2484$s, $0.6153$s, and $1.1705$s, respectively.


\begin{figure*}[t!]
\centering
	\begin{subfigure}[b]{0.245\textwidth}
	   \centering
		\includegraphics[width=1\linewidth]{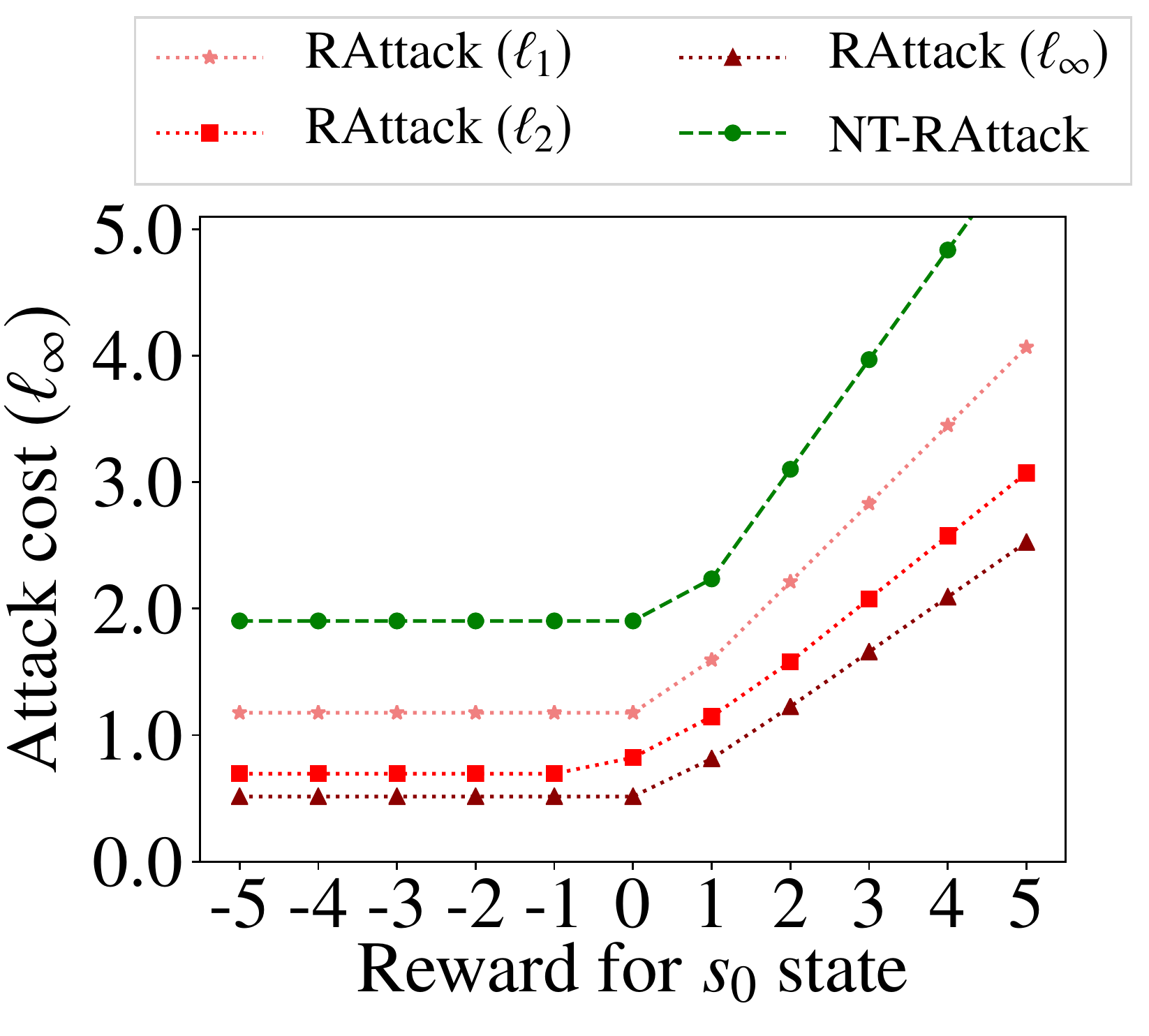}
		\caption{Vary $\overline{R}(s_0, .)$: $\widehat{R}$ attack}
		\label{fig:results.1}
	\end{subfigure}
	\begin{subfigure}[b]{0.245\textwidth}
	    \centering
		\includegraphics[width=1\linewidth]{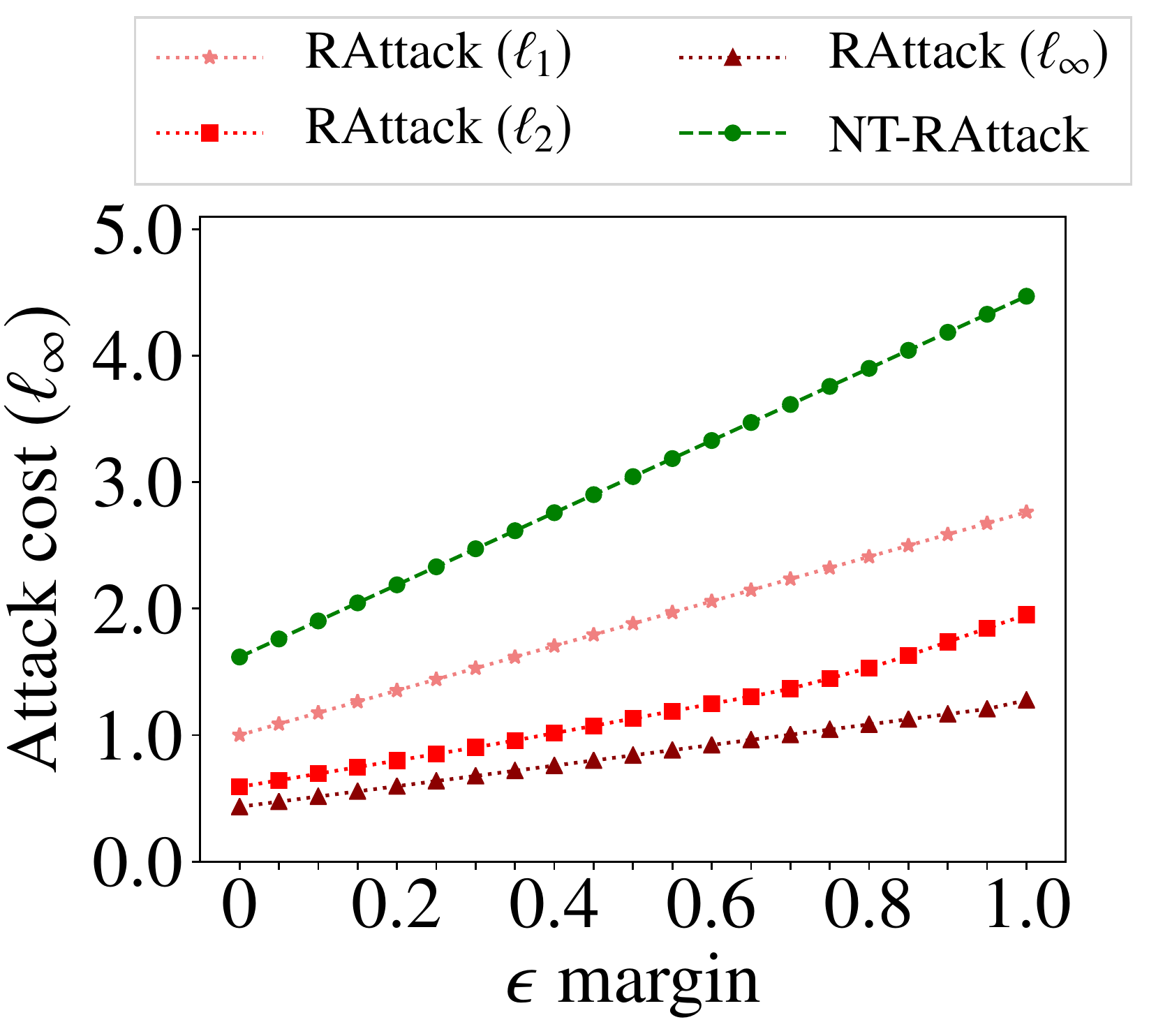}
		\caption{Vary $\epsilon$: $\widehat{R}$ attack}
		\label{fig:results.2}
	\end{subfigure}		
	\begin{subfigure}[b]{0.245\textwidth}
	    \centering
		\includegraphics[width=1\linewidth]{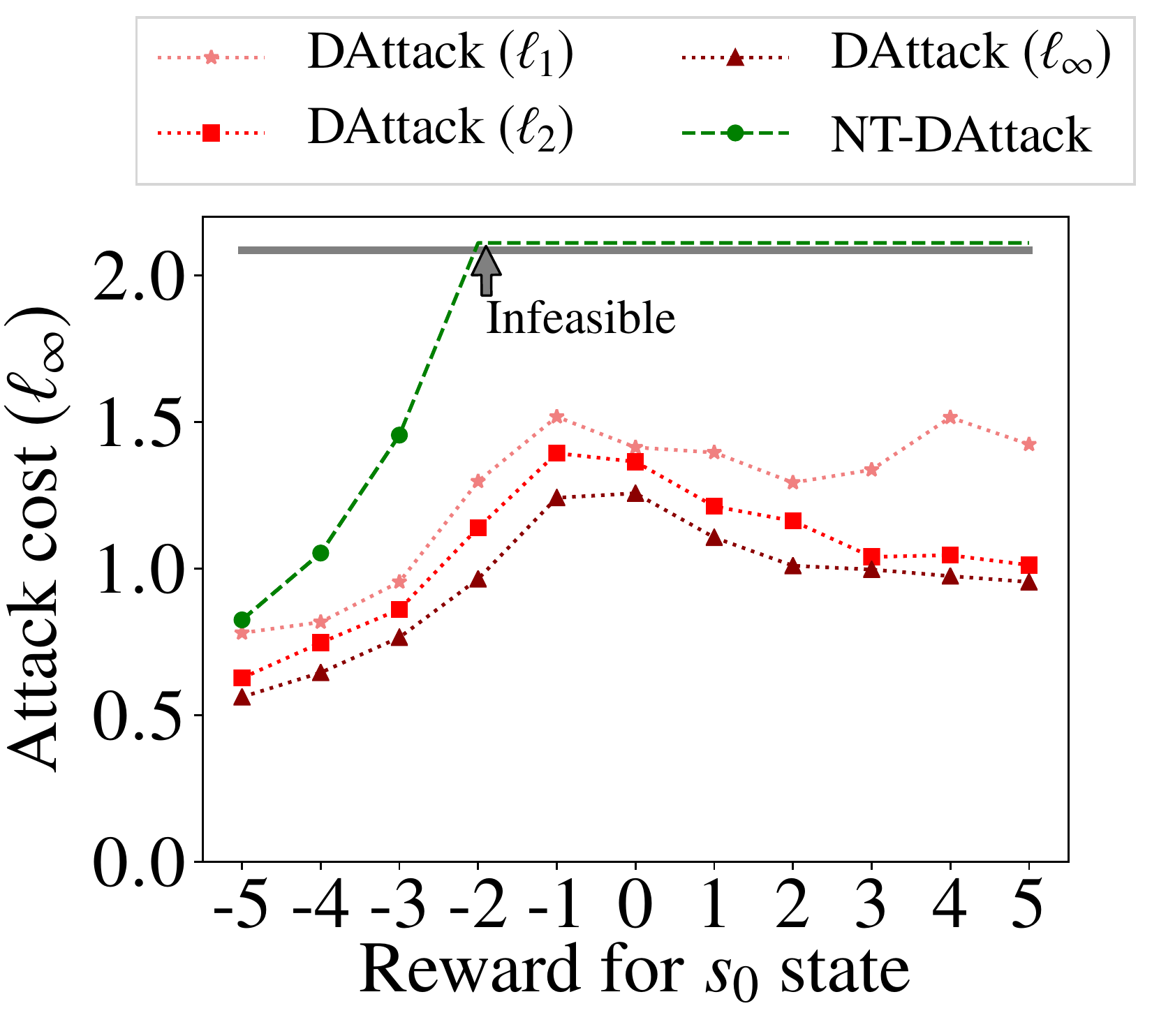}
		\caption{Vary $\overline{R}(s_0, .)$: $\widehat{P}$ attack}
		\label{fig:results.3}
	\end{subfigure}		
	\begin{subfigure}[b]{0.245\textwidth}
	    \centering
		\includegraphics[width=1\linewidth]{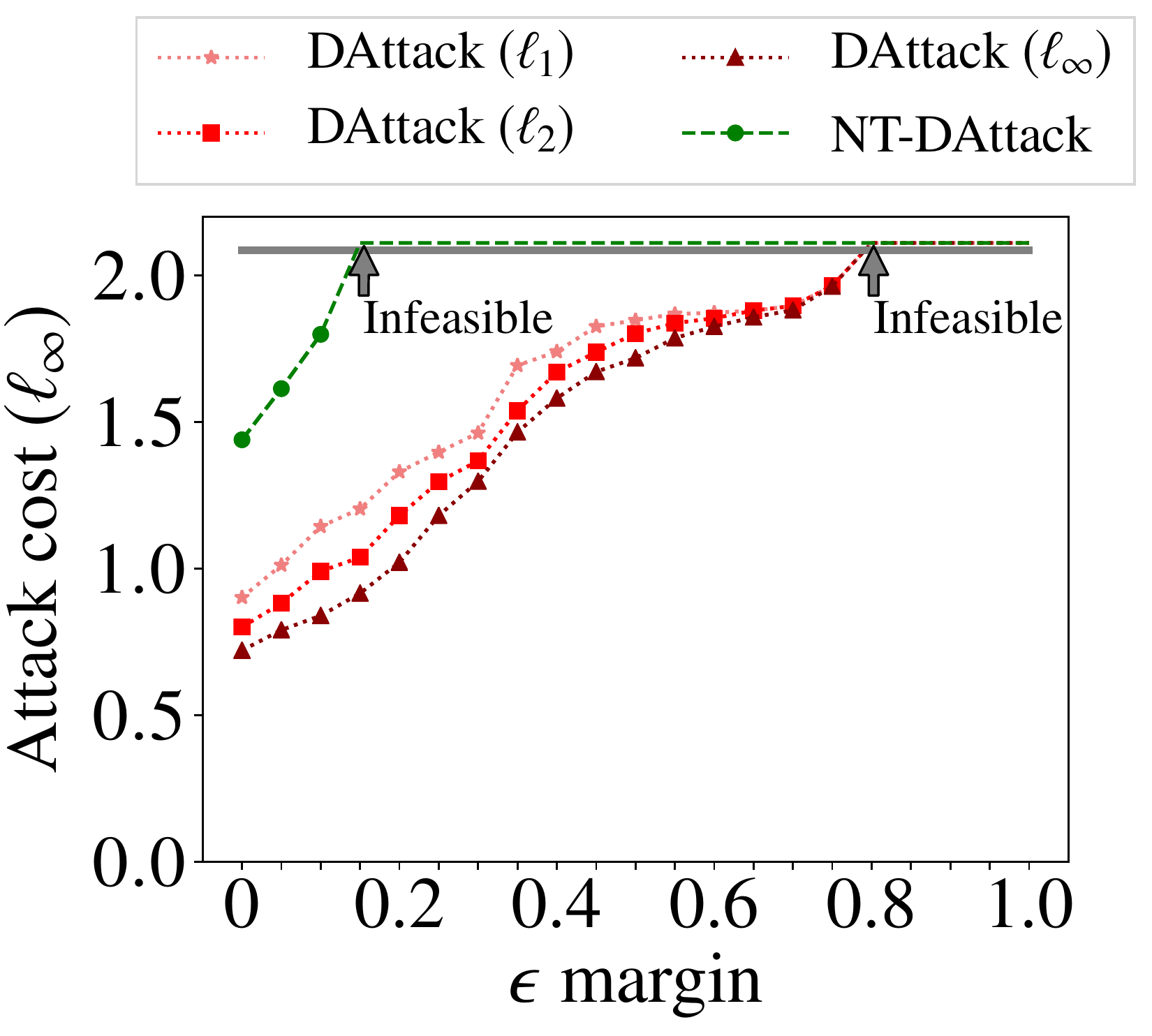}		
		\caption{Vary $\epsilon$:  $\widehat{P}$ attack}
		\label{fig:results.4}
	\end{subfigure}
   \caption{Results for poisoning attacks in the offline setting from Section~\ref{sec.offlineattacks}. (\textbf{a}, \textbf{b}) plots show results for attack on rewards and (\textbf{c}, \textbf{d}) plots show results for attack on dynamics. Details are in Section~\ref{sec.experiments.offline}.
   }
	\label{fig:results.offline}
	\vspace{-3mm}
\end{figure*}
%
%

\begin{figure*}[t!]
\centering
	\begin{subfigure}[b]{0.245\textwidth}
	   \centering
		\includegraphics[width=1\linewidth]{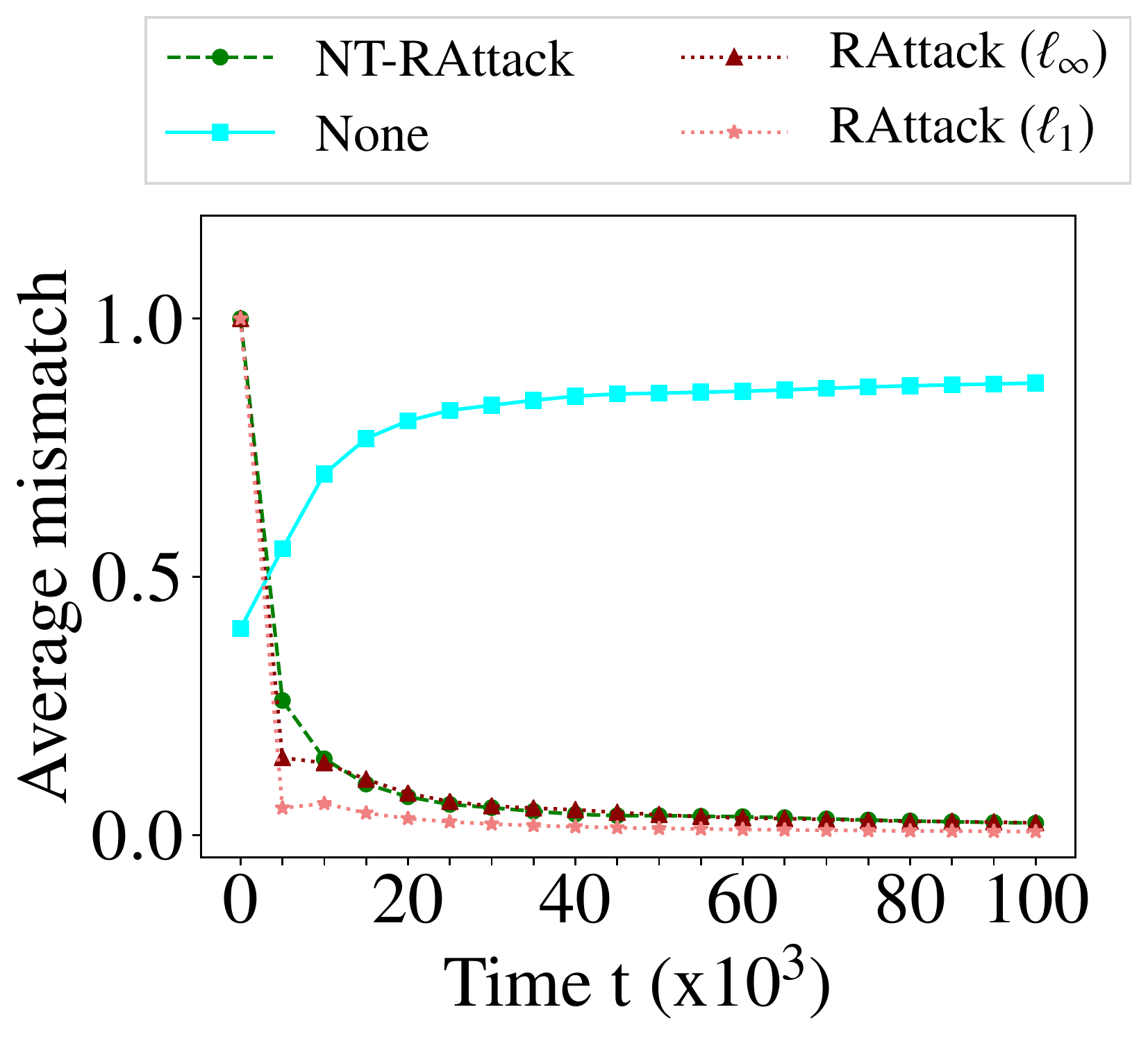}
		\caption{\avgmissm: $\widehat{R}$ attack}
		\label{fig:results.5}
	\end{subfigure}
	\begin{subfigure}[b]{0.245\textwidth}
	    \centering
		\includegraphics[width=1\linewidth]{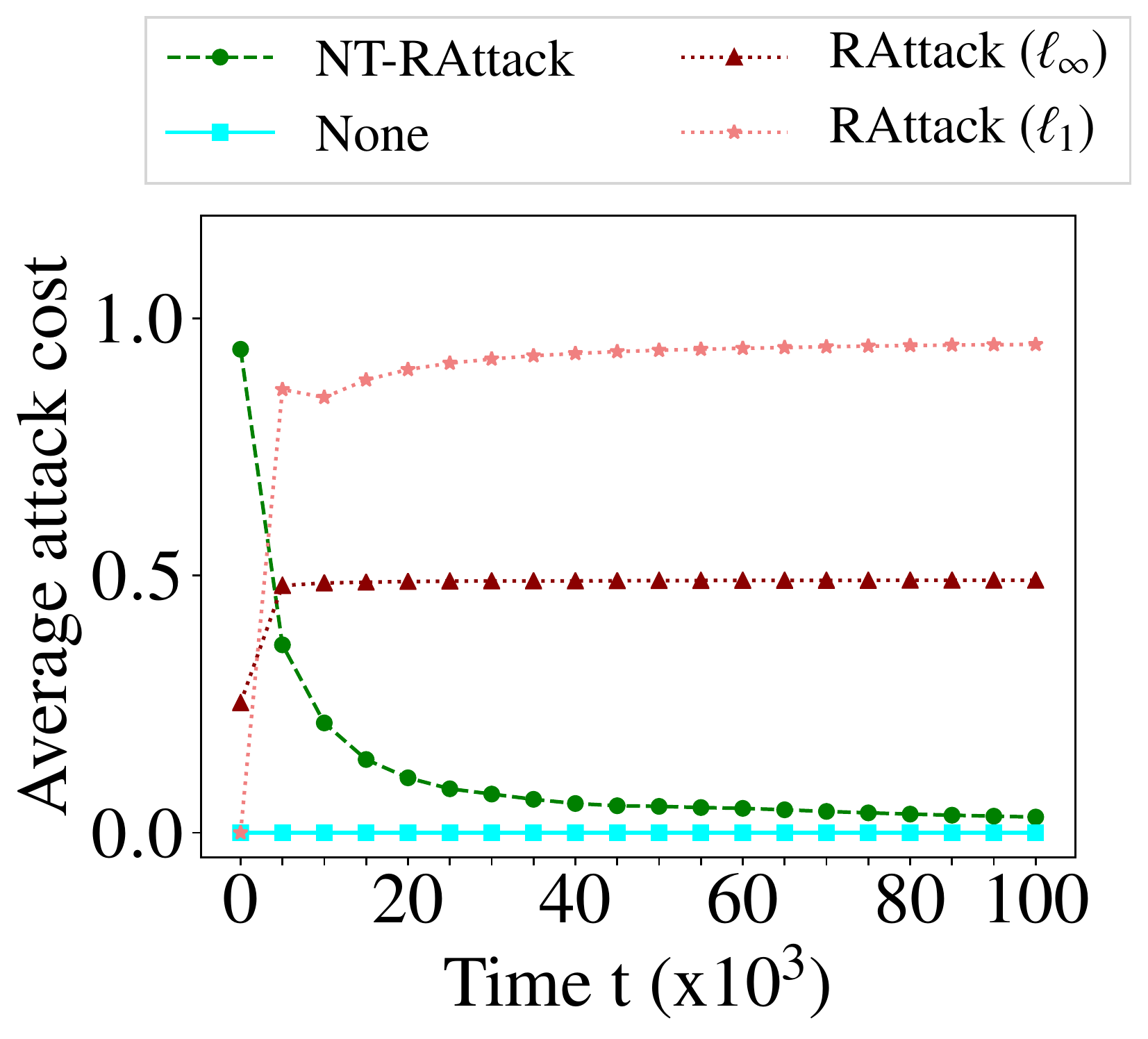}
		\caption{\avgcost: $\widehat{R}$ attack}
		\label{fig:results.6}
	\end{subfigure}		
	\begin{subfigure}[b]{0.245\textwidth}
	    \centering
		\includegraphics[width=1\linewidth]{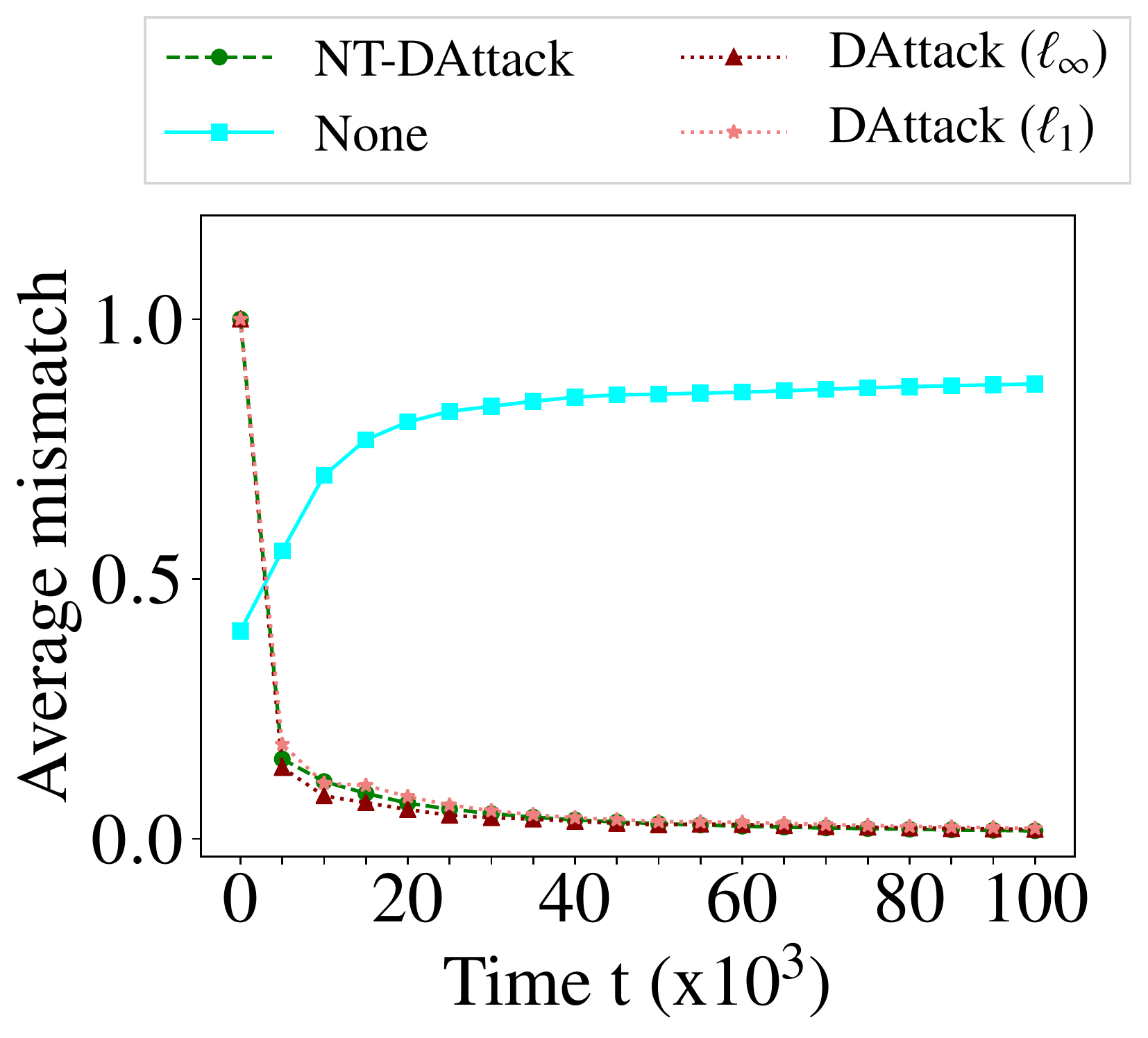}		
		\caption{\avgmissm: $\widehat{P}$ attack}
		\label{fig:results.7}
	\end{subfigure}		
	\begin{subfigure}[b]{0.245\textwidth}
	    \centering
		\includegraphics[width=1\linewidth]{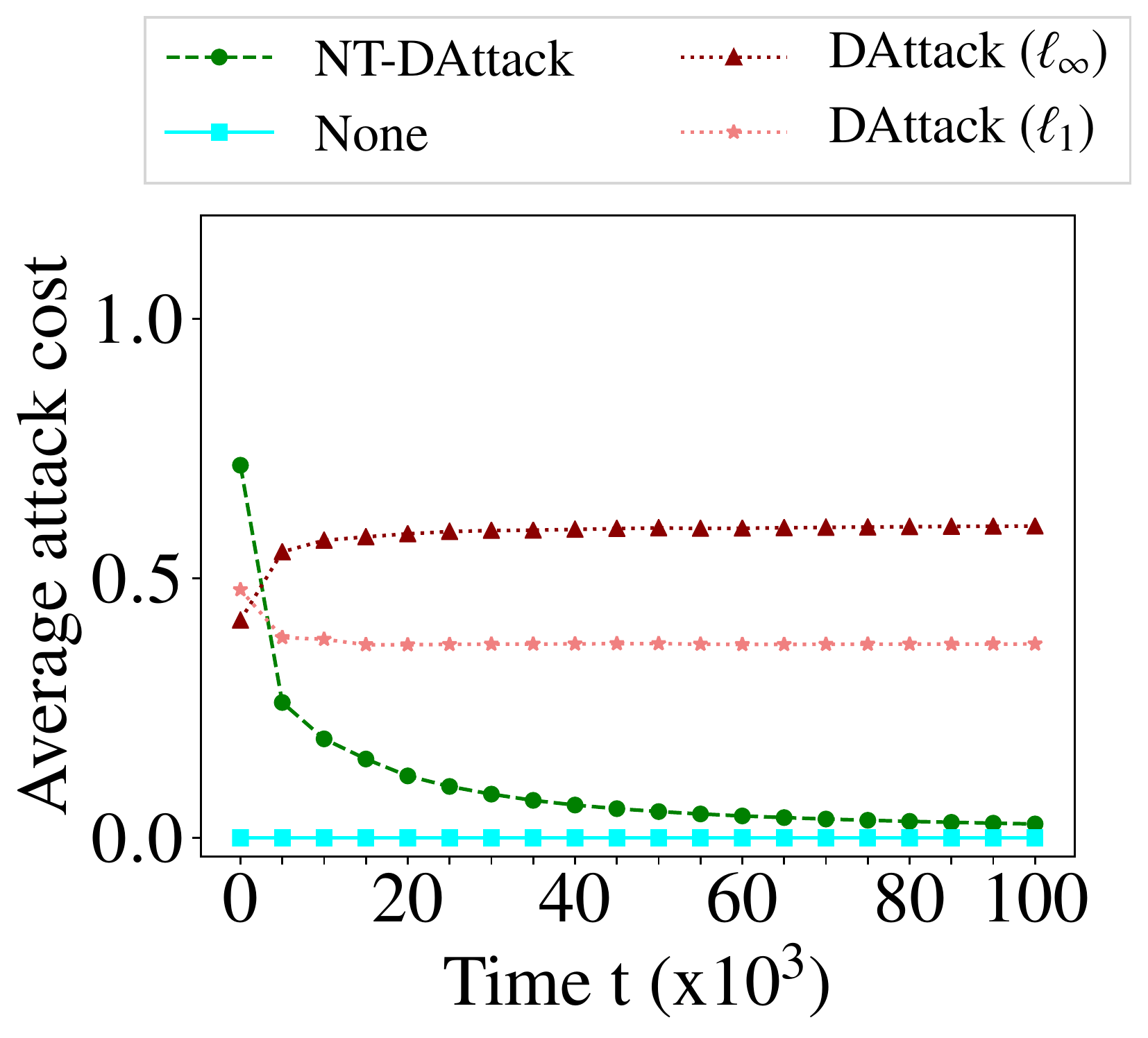}		
		\caption{\avgcost: $\widehat{P}$ attack}
		\label{fig:results.8}
	\end{subfigure}			
   \caption{Results for poisoning attacks in the online setting from Section~\ref{sec.onlineattacks}. (\textbf{a}, \textbf{b}) plots show results for attack on rewards and (\textbf{c}, \textbf{d}) plots show results for attack on dynamics. Details are in Section~\ref{sec.experiments.online}.
   }
	\label{fig:results.online}
   \vspace{-2mm}
\end{figure*}

\subsection{Attacks in the Online Setting: Setup and Results}\label{sec.experiments.online}
\looseness-1
For the online setting, we considered the following attack strategies: (i) reward attacks using \RAttack~(with $\ell_\infty$, $\ell_1$\emph{-norm}) and \NTRAttack; (ii) dynamics attacks using \DAttack~(with $\ell_\infty$, $\ell_1$\emph{-norm}) and \NTDAttack; and (iii) a default setting without adversary denoted as \NoAttack~where environment feedback is sampled from the original MDP $\overline{M}$.

In the experiments, we fix $\overline{R}(s_0, .)=-2.5$ and $\epsilon=0.1$; we plot the measure of the attacker's achieved goal in terms of \avgmissm~and attacker's cost in terms of \avgcost~for $\ell_1$\emph{-norm} measured over time $t$ (see Section~\ref{sec.formulation.online}). We consider an RL agent implementing the UCRL learning algorithm~\cite{auer2007logarithmic}, however the attacker does not use any knowledge of the agent's learning algorithm. The results are reported as an average of $20$ runs.

The results in Figure~\ref{fig:results.online} show that our proposed online attacks with \NTRAttack~and \NTDAttack~are highly effective: learner is forced to follow the target policy while the attacker's average cost is $o(1)$ (see Theorems~\ref{theorem.on.known.reward}, \ref{theorem.on.known.dyn}). In contrast, we can see that the online attacks with  \RAttack~and \DAttack~lead to high cost for the attacker, i.e., the cumulative cost is linear w.r.t. time as anticipated in Section~\ref{sec.on.ideas-definitions} (see discussions following Lemma~\ref{lemma.on.known.nontarget}).

\section{Conclusion}\label{sec.conclusions}
\looseness-1
We studied a security threat to reinforcement learning (RL) where an attacker poisons the environment, thereby forcing the agent into executing a target policy. Our work provides theoretical underpinnings of environment poisoning against RL along several new attack dimensions, including (i) adversarial manipulation of the transition dynamics, (ii) attack against RL agents maximizing average reward in undiscounted infinite horizon, and (iii) analyzing different attack costs for offline planning and online learning settings.

There are several promising directions for future work.
These include expanding the attack models (e.g., attacking rewards and transitions simultaneously) and broadening the set of attack goals (e.g., under partial specification of target policy). At the same time, relaxing the assumptions on the attacker knowledge of the underlying MDP could lead to more robust attack strategies. Another interesting future direction would be to make the studied attack models more scalable, e.g., applicable to continuous and large environments. Another interesting topic would be to devise attack strategies  against RL agents that use transfer learning approaches, especially in multi-agent RL systems, see \cite{da2019survey}.

While the paper provides a separate treatment for reward and transitions attack models, simultaneously attacking rewards and transitions might lead to more cost-effective solutions.  One possible way to formulate the optimization problem for the simultaneous attack model is to define the cost function as a weighted sum of the cost functions used in \eqref{prob.off.unc} and \eqref{prob.off.dyn}, and combine the corresponding constraints. 

\looseness-1While the experimental results demonstrate the effectiveness of the studied attack models, they do not reveal which types of learning algorithms are most vulnerable to the attack strategies studied in the paper. Further experimentation using a diverse set of the state of the art learning algorithms could reveal this, and provide some guidance in designing defensive strategies and novel RL algorithms robust to manipulations.

\section*{Acknowledgements}
Xiaojin Zhu is supported in part by NSF 1545481, 1623605, 1704117, 1836978 and the MADLab AF Center of Excellence FA9550-18-1-0166.



\bibliographystyle{icml2020}
\bibliography{main}

\iftoggle{longversion}{
\clearpage
\onecolumn
\appendix 
{\allowdisplaybreaks
\section{List of Appendices}\label{appendix:table-of-contents}
In this section we provide a brief description of the content provided in the appendices of the paper.   
\begin{itemize}
\item Appendix~\ref{sec.relatedwork} contains additional related work.
\item Appendix~\ref{appendix.sec.experiments} provides implementation details and additional experimental evaluation on a different environment.
\item Appendix~\ref{appendix_background} introduces some useful quantities and a lemma which are useful for proofs.
\item Appendix~\ref{appendix.off.general} contains proof of Lemma~\ref{lemma.using_neighbors} and some general results used in future proofs.
\item Appendix~\ref{appendix.off.rewards} contains proof of Theorem~\ref{theorem_off_unc} for offline attacks via poisoning rewards.
\item Appendix~\ref{appendix.off.dynamics} contains proof of Theorem~\ref{theorem_off_dyn} for offline attacks via poisoning dynamics and related discussions.
\item Appendix~\ref{appendix.on.general} contains proof of Lemma~\ref{lemma.on.known.nontarget}.
\item Appendix~\ref{appendix.on.rewards} contains proof of Theorem~\ref{theorem.on.known.reward} for online attacks via poisoning rewards.
\item Appendix~\ref{appendix.on.dynamics} contains proof of Theorem~\ref{theorem.on.known.dyn} for online attacks via poisoning dynamics.
\end{itemize}

\section{Additional Related Work} \label{sec.relatedwork}

\textbf{Test-time attacks against RL.} 
A growing body of contemporary works have studied test-time attacks against RL, in particular, on RL algorithms with neural network policies \cite{mnih2015human,schulman2015trust}. These attacks are typically done by adding noise in the observed state (e.g., a camera image) to fool the neural network policy into taking malicious actions~\cite{huang2017adversarial,DBLP:conf/ijcai/LinHLSLS17,tretschk2018sequential}. Different attack goals have been considered in these works, e.g., guiding the agent to some adversarial states or forcing agent to take actions that maximizes adversary's own rewards. 
Our work is technically quite different and is focused on training-time attacks where the goal is to force the agent to learn a target policy.


\textbf{Teaching an RL agent.} 
\looseness-1Poisoning attacks is mathematically equivalent to the formulation of machine teaching with teacher being the adversary~\cite{goldman1995complexity,singla2013actively,singla2014near,zhu2015machine,zhu2018overview,DBLP:conf/nips/ChenSAPY18,mansouri2019preference,DBLP:conf/nips/PeltolaCDK19,DBLP:conf/ijcai/DevidzeMH0S20}. In particular, there have been a number of recent works on teaching an RL agent via providing an optimized curriculum of demonstrations \cite{cakmak2012algorithmic,DBLP:conf/uai/WalshG12,hadfield2016cooperative,DBLP:conf/nips/HaugTS18,DBLP:conf/ijcai/KamalarubanDCS19,DBLP:conf/nips/TschiatschekGHD19,brown2019machine}. However, these works have focused on imitation-learning based RL agents who learn from provided demonstrations without any reward feedback \cite{osa2018algorithmic}. Given that we consider RL agents who find policies based on rewards, our work is technically very different from theirs. There is also a related literature on changing the behavior of an RL agent via \emph{reward shaping} \cite{ng1999policy,asmuth2008potential}; here the reward function is changed to only speed up the convergence of the learning algorithm while ensuring that the optimal policy in the modified environment is unchanged.
\section{Implementation Details and Additional Experiments (Section \ref{sec.experiments})}
\label{appendix.sec.experiments}
\looseness-1 In this section, we provide implementation details and report experimental  results on a different environment.\footnote{The source code is available at \url{https://github.com/adishs/icml2020_rl-policy-teaching_code}.}
\subsection{Implementation Details}
The optimal solutions to the problems \eqref{prob.off.unc}, \eqref{prob.on.known.reward_n1}, and \eqref{prob.on.known.dyn} can be computed efficiently using standard optimization techniques. In the source code, the solvers for these optimization problems are implemented as the following functions:
\begin{itemize}
    \item problem~\eqref{prob.off.unc} in the function \texttt{general\_attack\_on\_reward()}, see \texttt{teacher.py}
    \item problem~\eqref{prob.on.known.reward_n1} in the function \texttt{non\_target\_attack\_on\_reward()}, see \texttt{teacher.py}
    \item problem~\eqref{prob.on.known.dyn} in the function \texttt{non\_target\_attack\_on\_dynamics()}, see  \texttt{teacher.py}
\end{itemize}

Our approach for solving problem \eqref{prob.off.dyn} is implemented in the function \texttt{general\_attack\_on\_dynamics()}, see  \texttt{teacher.py}. Details are provided below:
\begin{itemize}
    \item As a first step, we use a simple heuristic to obtain a pool of transition kernels $\widetilde{P}$  by perturbations of $\overline{P}$ that increase the average reward of $\targetpi$. Here, the transition kernel $\widetilde{P}$ differs from $\overline{P}$ only for the actions taken by the target policy, i.e., $(s, \targetpi(s)) \ \forall s \in S$. This pool is created in the function \texttt{generate\_pool()}.
    \item As the second step, we take each of $\widetilde{P}$'s from this pool as as input to problem \eqref{prob.on.known.dyn} instead of $\overline{P}$ which in turn gives us a corresponding pool of solutions. Note here the problem \eqref{prob.on.known.dyn} modifies only the actions which are not taken by the target policy, i.e., $(s, a) \ \forall s \in S, a \neq \targetpi(s)$. This is done in the function \texttt{solve\_pool()}.
    \item As the final step, we pick a solution from this pool of solutions with the minimal cost.  This is done in the function \texttt{get\_P\_with\_smallest\_norm()}.
\end{itemize}

\subsection{Additional Experiments}
We perform additional numerical simulations on an environment represented as an MDP with nine states and two actions per state, see Figure~\ref{fig:environment.grid} for details. This environment is inspired from a navigation task and is slightly more complex than the environment in Figure~\ref{fig:environment}.

\begin{minipage}{\linewidth}
    \makebox[\linewidth]{
    \includegraphics[width=0.6\textwidth]{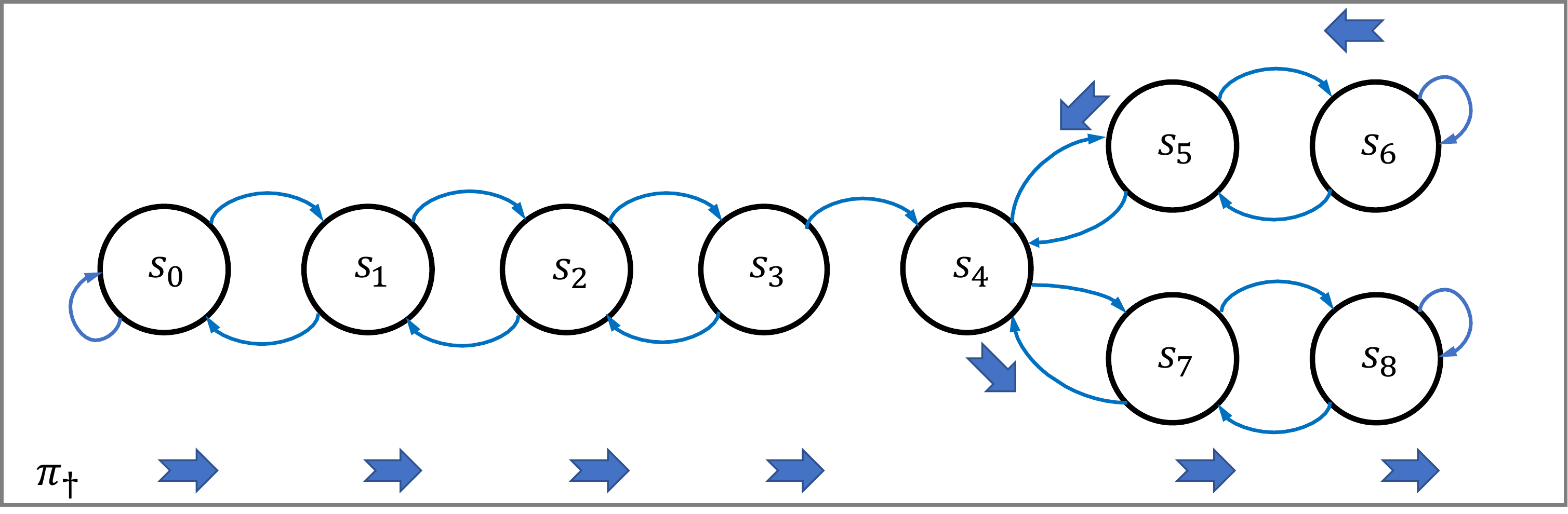}
  	}
  	\captionof{figure}{The environment has $|S| = 9$ states and $|A|=2$ actions per state as illustrated. The original reward function $\overline{R}$ is action independent and has the following values: $\overline{R}(s_1,.) = \overline{R}(s_2,.) = \overline{R}(s_3,.) = -2.5$,  $\overline{R}(s_4,.) = \overline{R}(s_5,.) = 1.0$, $\overline{R}(s_6,.) = \overline{R}(s_7,.) = \overline{R}(s_8,.) = 0$, and the reward of the state $s_0$ given by $\overline{R}(s_0,.)$ will be varied in experiments. With probability $0.9$, the actions succeed in navigating the agent as shown on arrows; with probability $0.1$ the agent's next state is sampled randomly from the set $S$. The target policy $\targetpi$ is to take actions as shown with bold arrows in the illustration. 
  	}
	\label{fig:environment.grid}
\end{minipage}
%

\begin{figure*}[h!]
\centering
	\begin{subfigure}[b]{0.245\textwidth}
	   \centering
		\includegraphics[width=1\linewidth]{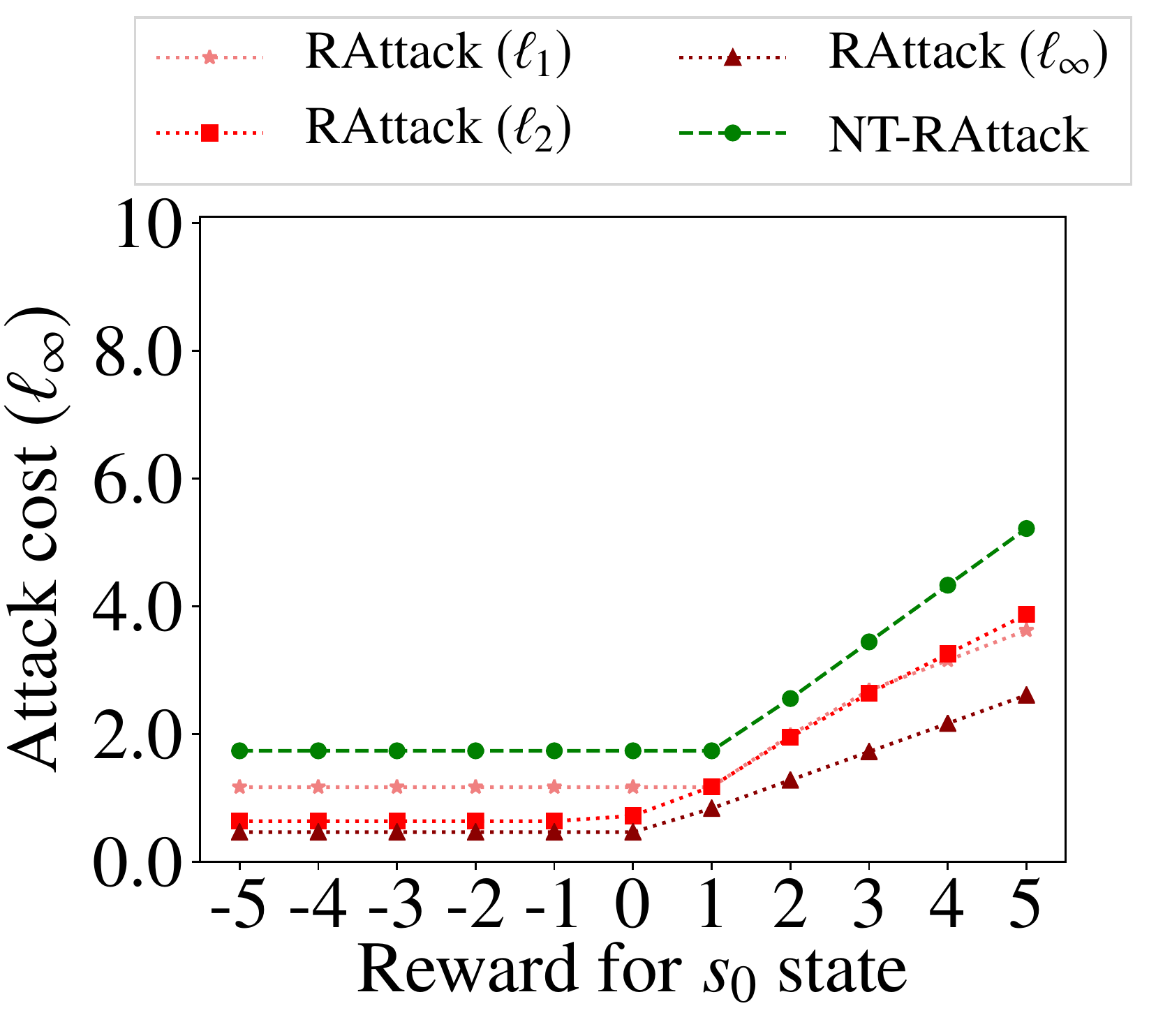}
		\caption{Vary $\overline{R}(s_0, .)$: $\widehat{R}$ attack}
		\label{fig:results.1.grid}
	\end{subfigure}
	\begin{subfigure}[b]{0.245\textwidth}
	    \centering
		\includegraphics[width=1\linewidth]{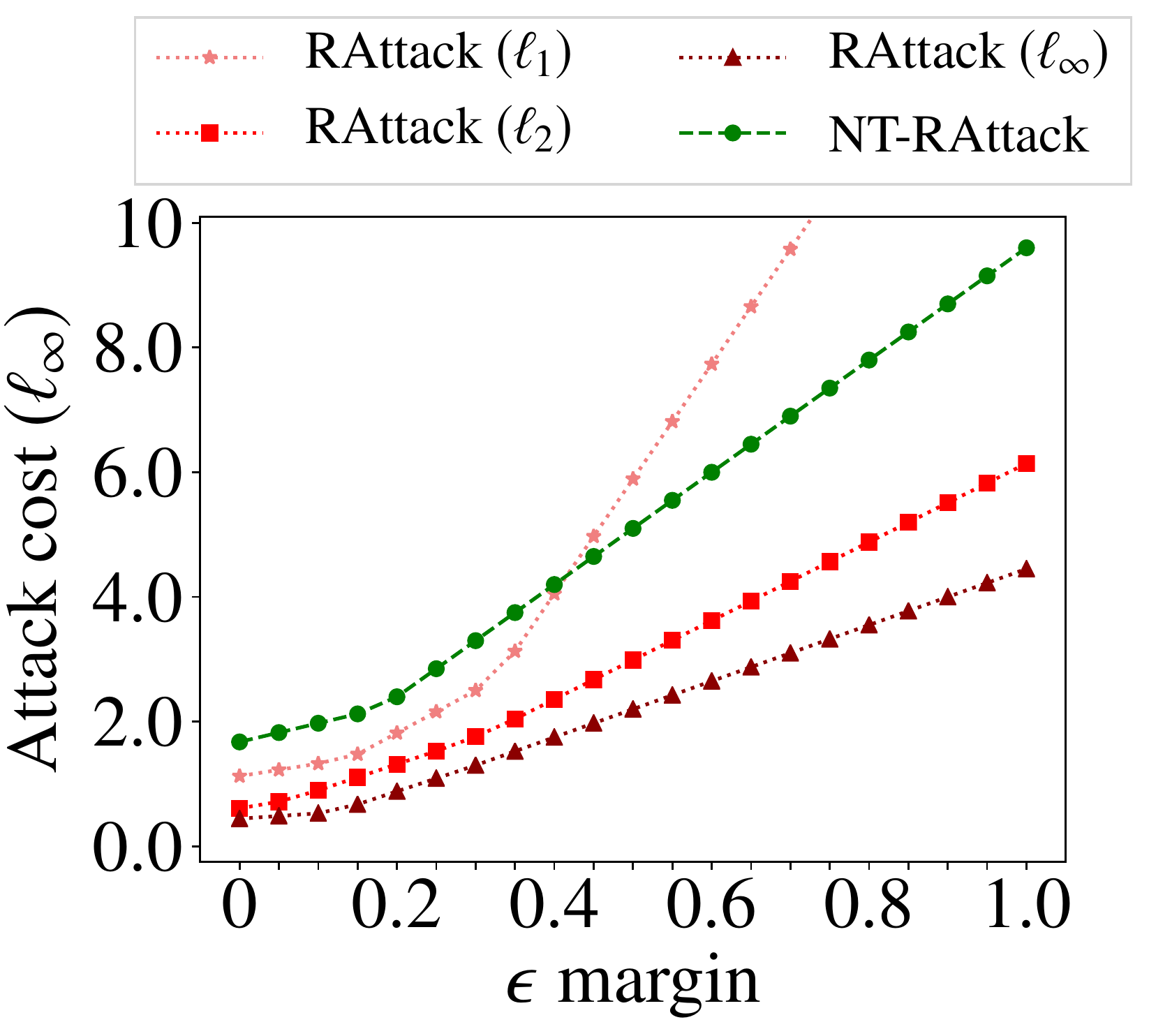}
		\caption{Vary $\epsilon$: $\widehat{R}$ attack}
		\label{fig:results.2.grid}
	\end{subfigure}		
	\begin{subfigure}[b]{0.245\textwidth}
	    \centering
		\includegraphics[width=1\linewidth]{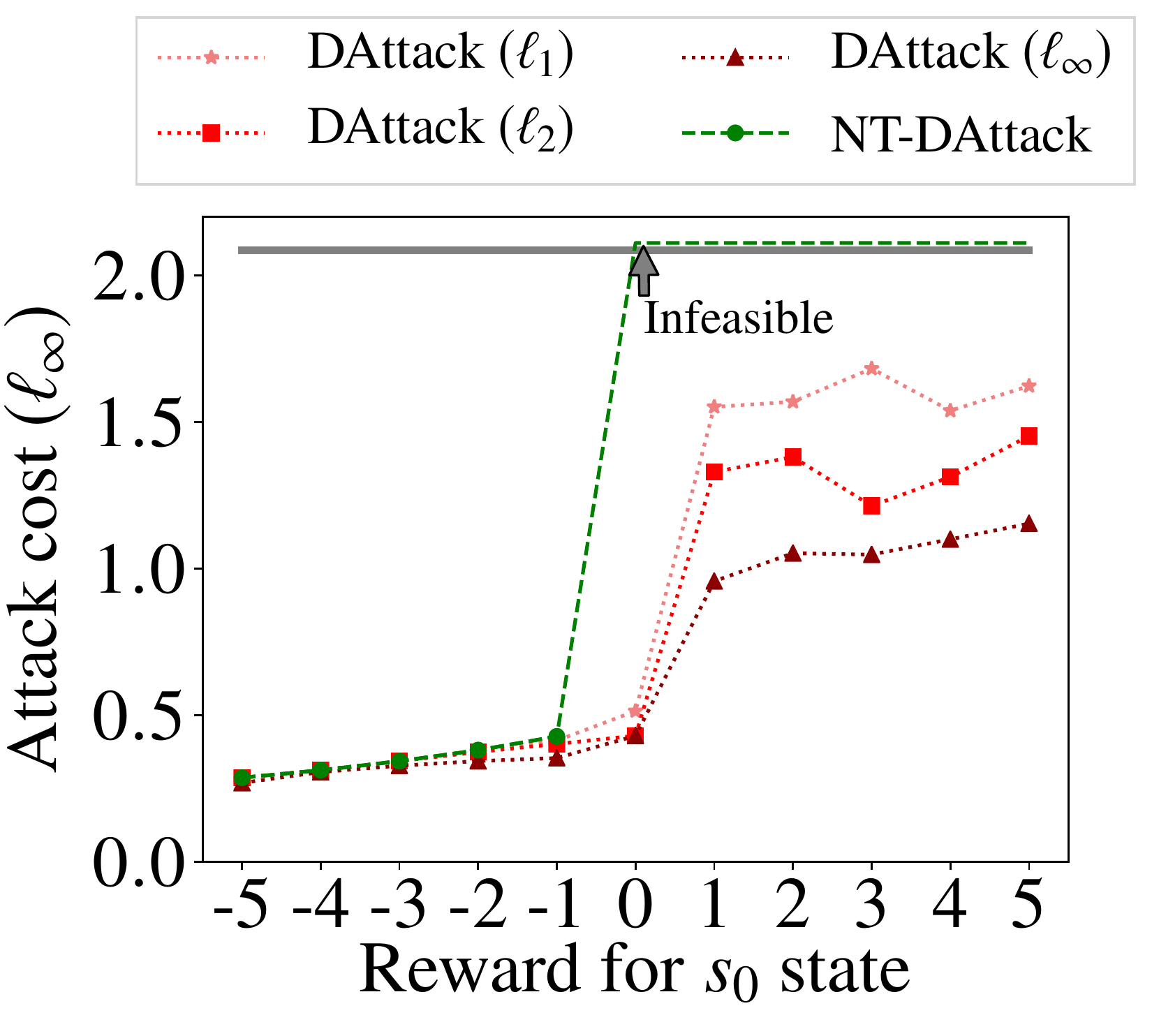}
		\caption{Vary $\overline{R}(s_0, .)$: $\widehat{P}$ attack}
		\label{fig:results.3.grid}
	\end{subfigure}		
	\begin{subfigure}[b]{0.245\textwidth}
	    \centering
		\includegraphics[width=1\linewidth]{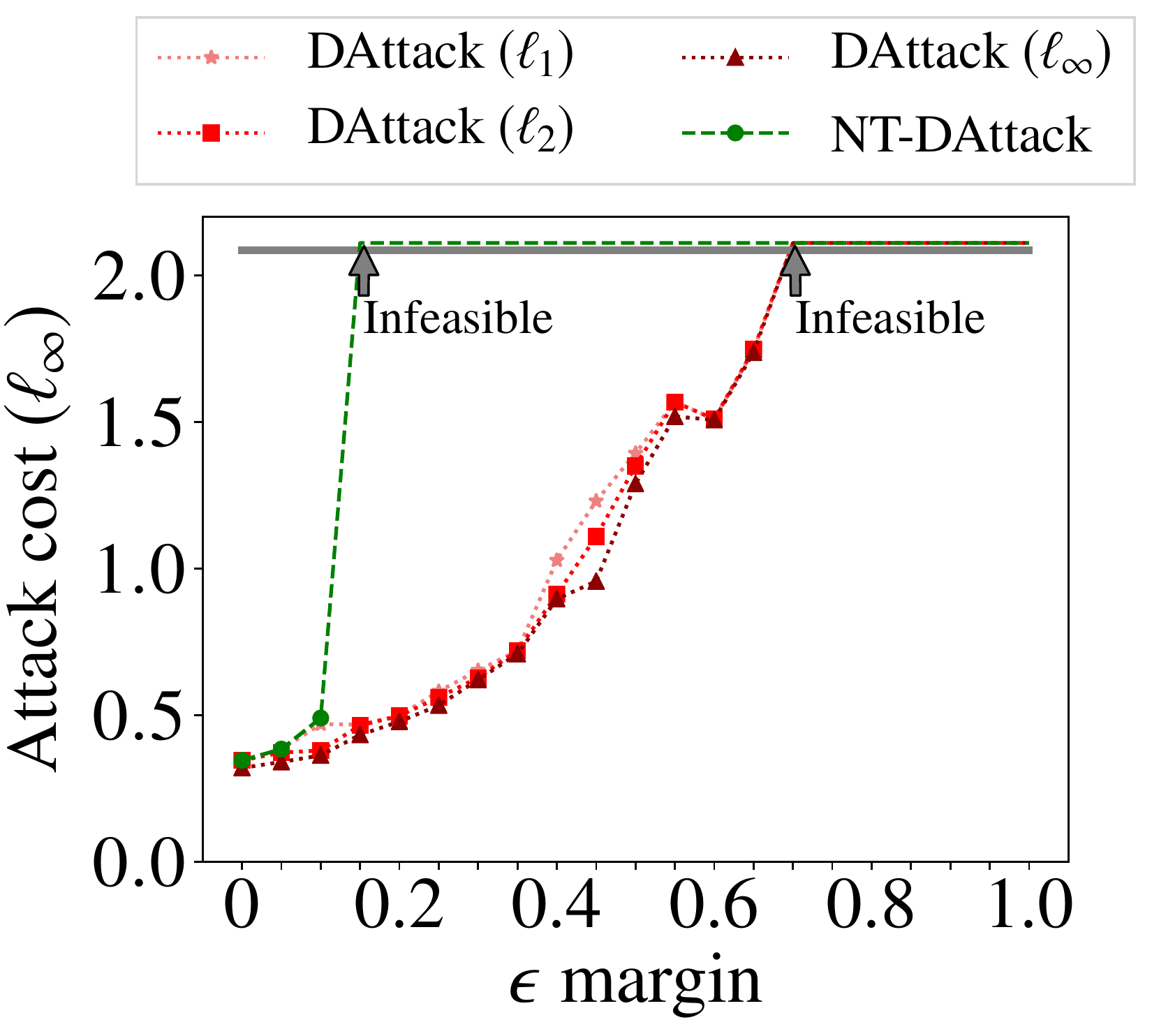}		
		\caption{Vary $\epsilon$:  $\widehat{P}$ attack}
		\label{fig:results.4.grid}
	\end{subfigure}
   \caption{Results for poisoning attacks in the offline setting from Section~\ref{sec.offlineattacks}. (\textbf{a}, \textbf{b}) plots show results for attack on rewards and (\textbf{c}, \textbf{d}) plots show results for attack on dynamics.
   }
	\label{fig:results.offline.grid}
	\vspace{-2mm}
\end{figure*}
%
%

\begin{figure*}[h!]
\centering
	\begin{subfigure}[b]{0.245\textwidth}
	   \centering
		\includegraphics[width=1\linewidth]{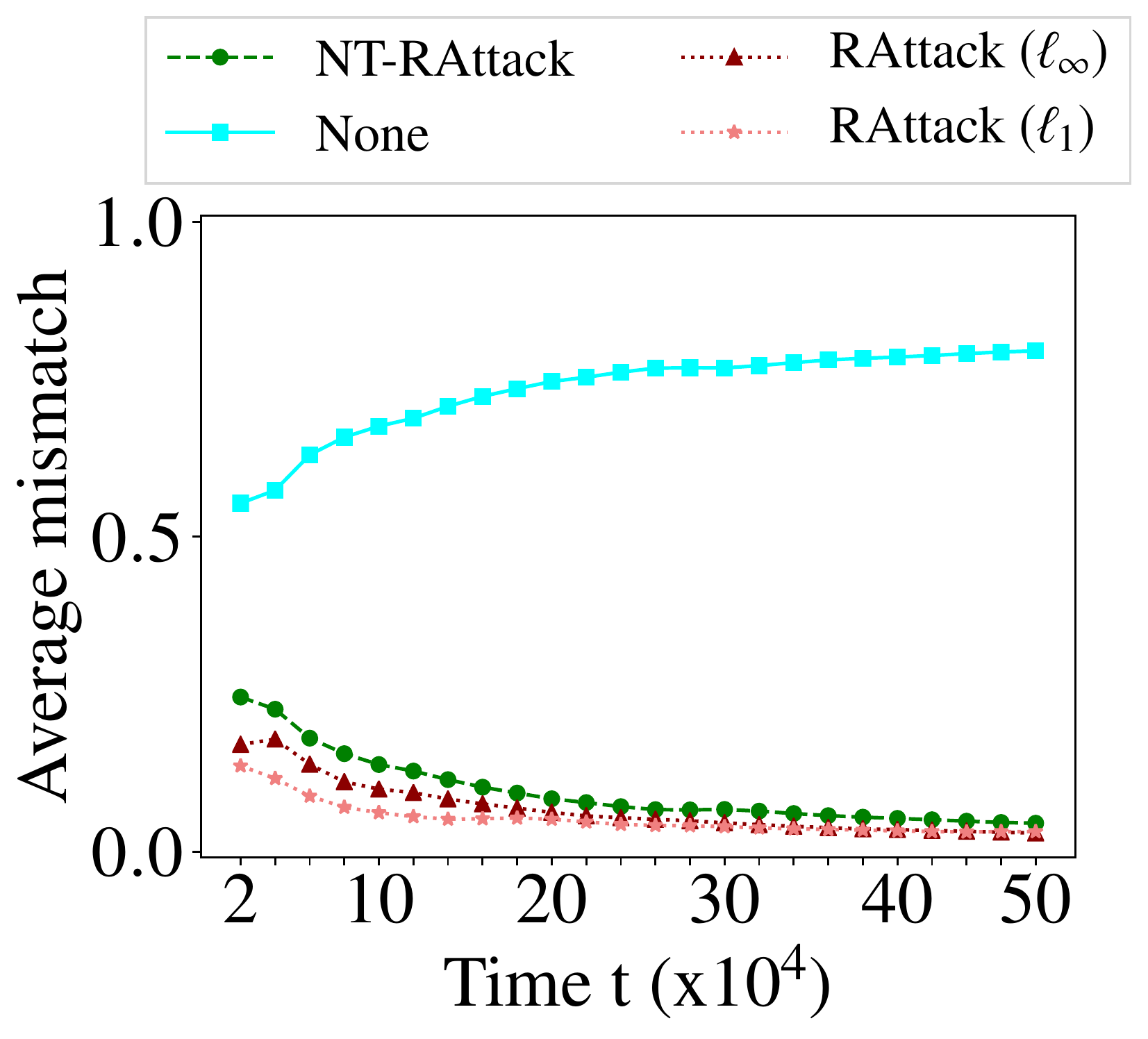}
		\caption{\avgmissm: $\widehat{R}$ attack}
		\label{fig:results.5.grid}
	\end{subfigure}
	\begin{subfigure}[b]{0.245\textwidth}
	    \centering
		\includegraphics[width=1\linewidth]{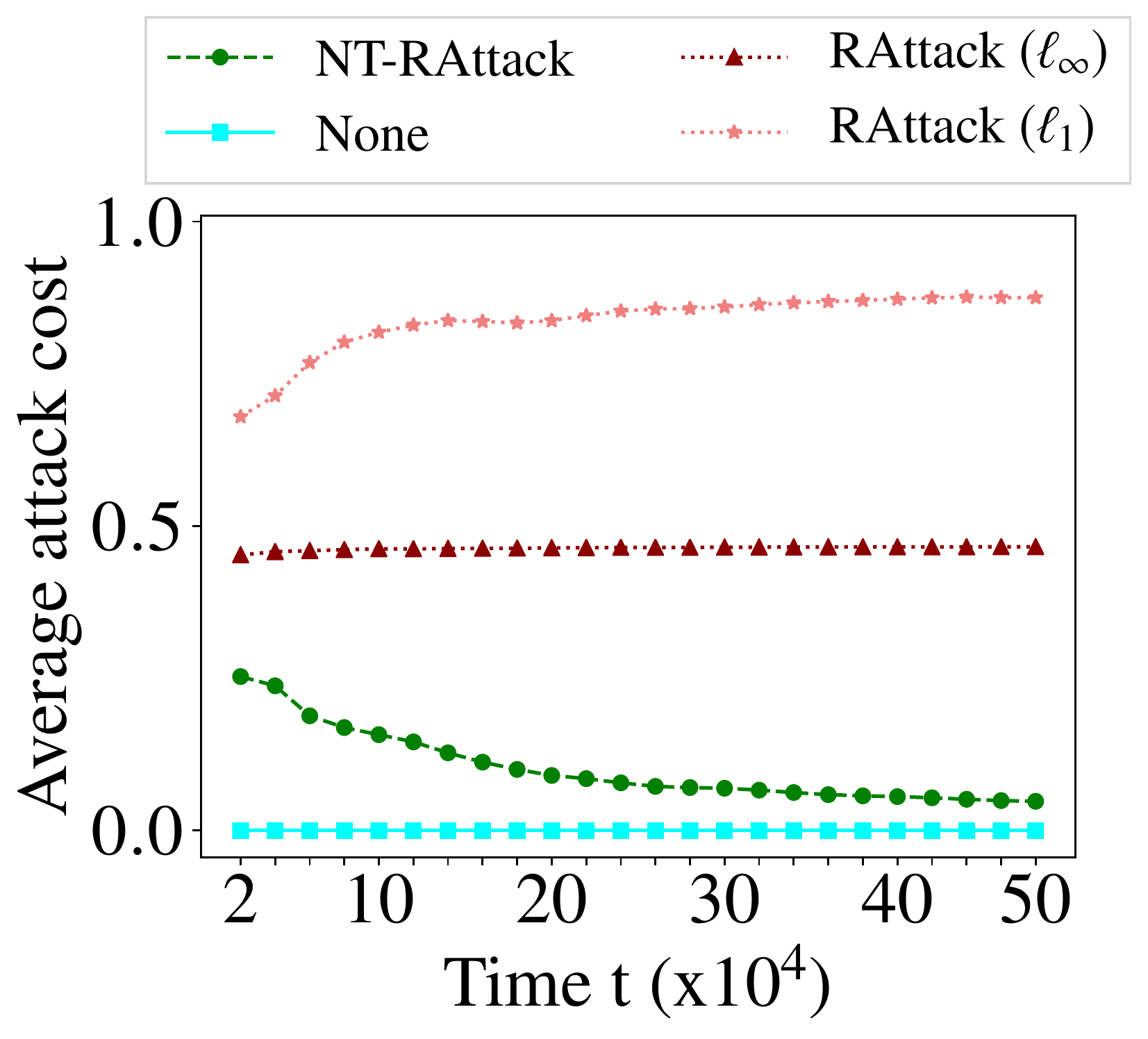}
		\caption{\avgcost: $\widehat{R}$ attack}
		\label{fig:results.6.grid}
	\end{subfigure}		
	\begin{subfigure}[b]{0.245\textwidth}
	    \centering
		\includegraphics[width=1\linewidth]{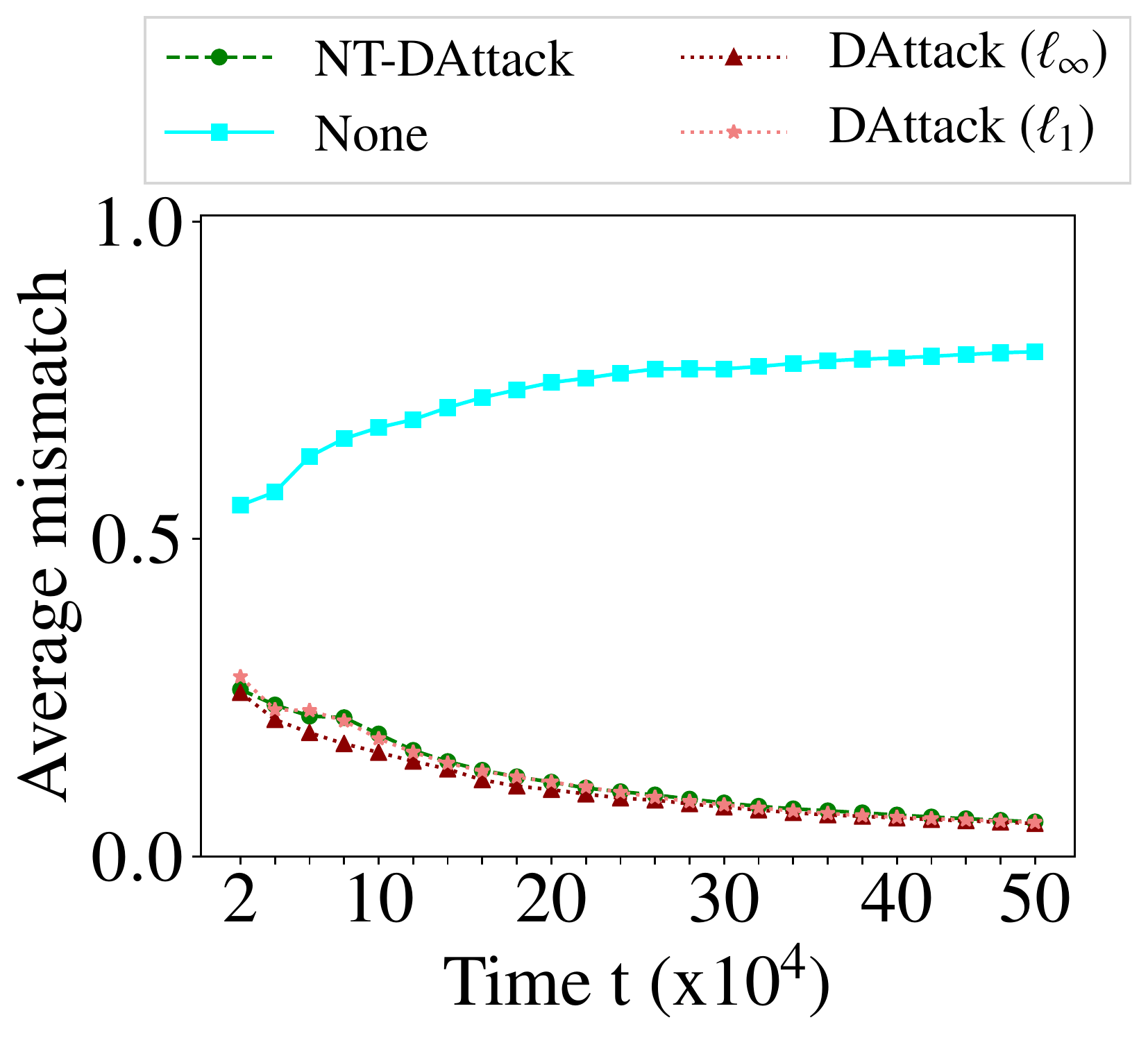}		
		\caption{\avgmissm: $\widehat{P}$ attack}
		\label{fig:results.7.grid}
	\end{subfigure}		
	\begin{subfigure}[b]{0.245\textwidth}
	    \centering
		\includegraphics[width=1\linewidth]{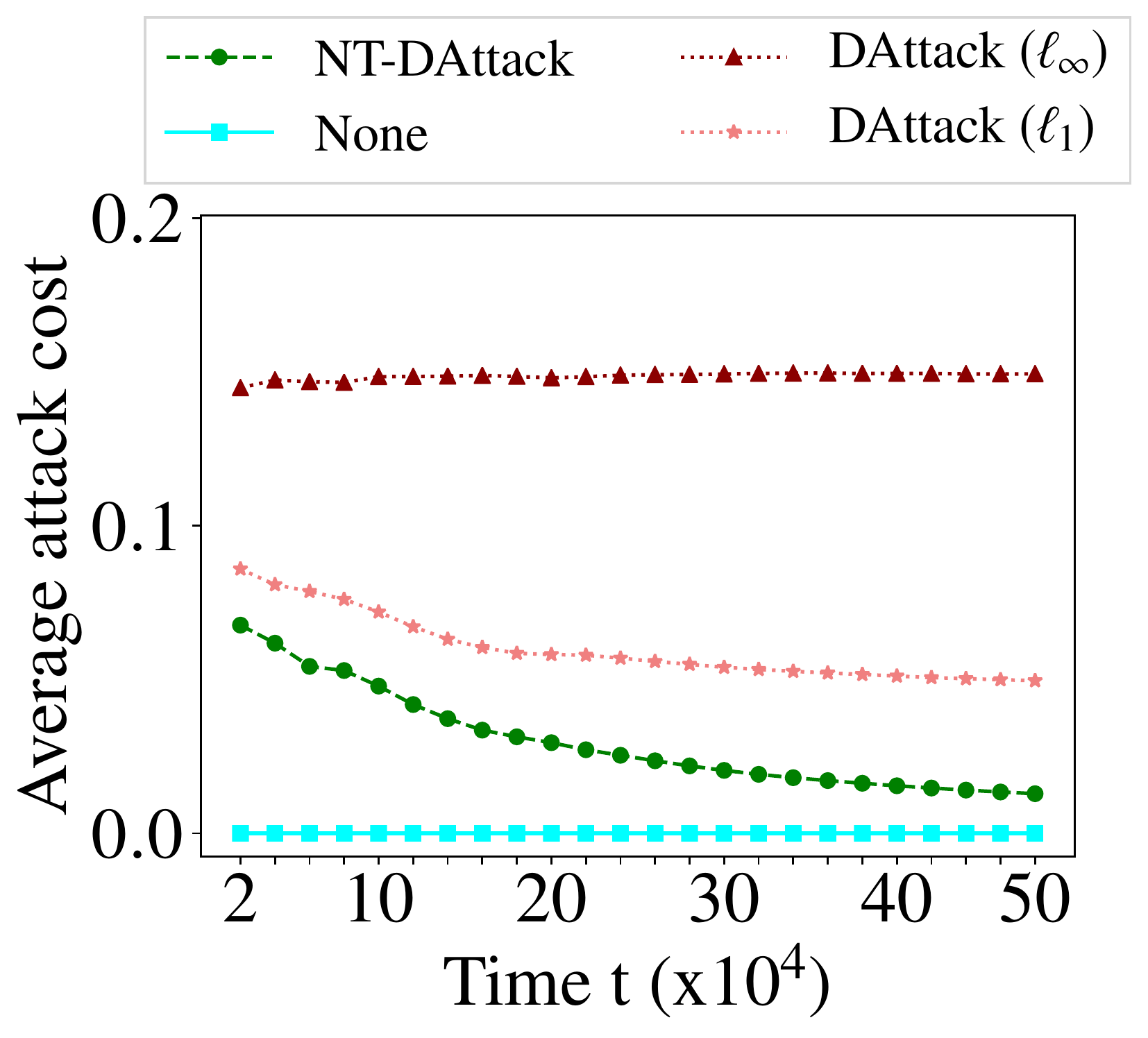}		
		\caption{\avgcost: $\widehat{P}$ attack}
		\label{fig:results.8.grid}
	\end{subfigure}			
   \caption{Results for poisoning attacks in the online setting from Section~\ref{sec.onlineattacks}. (\textbf{a}, \textbf{b}) plots show results for attack on rewards and (\textbf{c}, \textbf{d}) plots show results for attack on dynamics.
   }
	\label{fig:results.online.grid}
	\vspace{-2mm}
\end{figure*}

The experimental setup and attack strategies used for these additional experiments are exactly same as that used in Section~\ref{sec.experiments} for both the offline and the online settings. To recall, for the offline setting, we vary $\overline{R}(s_0, .) \in [-5, 5]$ and vary $\epsilon$ margin $\in [0, 1]$. We use $\ell_\infty$\emph{-norm} in the measure of attack cost. The results are reported as an average of $10$ runs. For the online setting, we fix $\overline{R}(s_0, .)=-2.5$ and $\epsilon=0.1$. We plot the measure of the attacker's achieved goal in terms of \avgmissm~and attacker's cost in terms of \avgcost~for $\ell_1$\emph{-norm} measured over time $t$. The results are reported as an average of $20$ runs.

While the scales of cost and mismatch are different, the key takeaway results from Figure~\ref{fig:results.offline.grid} (offline setting) and Figure~\ref{fig:results.online.grid} (online setting) are same as discussed in Section~\ref{sec.experiments.offline} and Section~\ref{sec.experiments.online} respectively.

\section{Background}
\label{appendix_background}
Throughout many proofs, we will use relative values defined in average reward reinforcement learning. Relative values or bias values of policy $\pi$ are defined as following\footnote{
This limit assumes that all policies are aperiodic. For periodic policies, we need to use the Cesaro limit \cite{DBLP:journals/ml/Mahadevan96}.
}:
\begin{gather*}
Q^\pi(s, a) = \lim_{N\to \infty} \expct{\sum^{N-1}_{t=0} \Big(R(s_t, a_t) - \rho^\pi\Big)| s_0=s, a_0=a, \pi},
\end{gather*}{}
where $s_0$ and $a_0$ are the initial state and we follow policy $\pi$ after taking action $a_0$ in state $s_0$. We will often refer to relative values as $Q$ values. These values satisfy the following recurrence equation
\begin{gather*}
Q^\pi(s, a) = R(s, a) - \rho^\pi + \sum_{s'\in \cS} P(s, a, s')Q^\pi(s', \pi(s')).
\end{gather*}{}
By definition we have $V^\pi(s) = Q^\pi(s, \pi(s))$. Based on Corollary 8.2.7 in \cite{Puterman1994}, these values can be calculated by solving set of equations
\begin{gather}
    \label{calcualte_v_1}
 V^\pi(s) = R(s, \pi(s)) - \rho^\pi + \sum_{s' \in \cS} P(s, \pi(s), s')V^\pi(s')\\
    \label{calcualte_v_2}
    \sum_{s\in\cS} \mu^\pi(s)V^\pi(s) = 0
\end{gather}{}
for $V^\pi(s)$ values and setting
\begin{gather}
\label{calcualte_q_1}
Q^\pi(s, a) = R(s, a) - \rho^\pi + \sum_{s'\in \cS} P(s, a, s')V^\pi(s').
\end{gather}{}

We rely on the following result of \citet{even2005experts}.  
\begin{lemma}\label{lemma_qrho_relate}
(Lemma 7 in \cite{even2005experts}) For two policies $\pi$ and $\pi'$ we have:
$$
\rho^\pi - \rho^{\pi'} = 
\sum_{s\in\cS}\mu^{\pi'}(s)\big(Q^\pi(s, \pi(s)) - Q^\pi(s, \pi'(s))\big).
$$
\end{lemma}{}



\section{Proofs for Offline Attacks: Lemma~\ref{lemma.using_neighbors} (Section
\ref{sec.off.ideas-definitions})}
\label{appendix.off.general}

We prove Lemma~\ref{lemma.using_neighbors} through several intermediate results. The first one is a direct consequence of Lemma \ref{lemma_qrho_relate}.
\begin{corollary}
\label{gain_diff_neighbor}
For any policy $\pi$ and its neighbor policy $\neighbor{\pi}{s}{a}$ we have:
$$
\rho^\pi - \rho^{\neighbor{\pi}{s}{a}} = 
\mu^{\neighbor{\pi}{s}{a}}(s)\big(Q^\pi(s, \pi(s)) - Q^\pi(s,a)\big).
$$
\end{corollary}{}

 Next, we provide a sufficient condition for a policy $\pi$ to be uniquely optimal. 
 \begin{lemma}
\label{single_optimal}
If we have $\rho^\pi \ge \rho^{\neighbor{\pi}{s}{a}} + \epsilon$ for every state $s$ and action $a \ne \pi(s)$, and $\epsilon > 0$, then $\pi$ is the only optimal policy.
\end{lemma}
\begin{proof}
Let arbitrary $s\in\cS$ and $a\in\cA$ be such that $a\ne\pi(s)$. Based on corollary \ref{gain_diff_neighbor}, we have
\begin{equation}
\label{has_max_q}
Q^\pi(s, \pi(s)) - Q^\pi(s,a) = \frac{\rho^\pi - \rho^{\neighbor{\pi}{s}{a}}}{\mu^{\neighbor{\pi}{s}{a}}(s)} > 0.
\end{equation}{}
Note that in this paper we are focusing on recurrent MDPs, thus, we know $\mu^{\neighbor{\pi}{s}{a}}(s) > 0$. Now let $\pi'\ne\pi$ be a policy with $\pi'(s') = a' \ne \pi(s')$. Using (\ref{has_max_q}) we have
\begin{equation*}
\rho^\pi - \rho^{\pi'} = 
\sum_{s\in\cS}\mu^{\pi'}(s)\big(Q^\pi(s, \pi(s)) - Q^\pi(s, \pi'(s))\big) \ge \mu^{\pi'}(s')\big(Q^\pi(s', a') - Q^\pi(s',\pi(s'))\big) > 0
\end{equation*}{}
We again used the fact that MDP is recurrent to say $\mu^{\pi'}(s') > 0$.
\end{proof}{}


\textbf{Proof of Lemma~\ref{lemma.using_neighbors}
}
\begin{proof}
The necessity of the condition follows directly from the definition of $\epsilon$-robust policies. Let us focus on its sufficiency. 

Consider deterministic policies $\pi$, and denote the Hamming distance between two policies $\pi_1$ and $\pi_2$ by $D_H(\pi_1, \pi_2)$, i.e., $D_H(\pi_1, \pi_2) = \sum_{s \in \cS} \ind{\pi_1(s) \ne \pi_2(s)}$ where $\ind{.}$ denotes the indicator function. 
Assume that the condition of the proposition holds for policy $\pi^*$, i.e., that $\rho^{\pi^*} \ge \rho^{\pi_1} + \epsilon$ for all $\pi_1$ s.t. $D_H(\pi_1, \pi^*) = 1$. 

Lemma ~\ref{single_optimal} implies that $\pi^*$ is uniquely optimal. Now, consider policy $\pi_k$ s.t. $D_H(\pi_k, \pi^*) = k > 1$. Since $\pi^*$ is (uniquely) optimal and the MDP is recurrent ($\mu^{\pi^*}(s) > 0$), we have that
\begin{align*}
\rho^{\pi^*} - \rho^{\pi_k} = \sum_{s \in \cS} \mu^{\pi^*}(s) \cdot [Q^{\pi_k}(s, \pi^*(s)) - Q^{\pi_k}(s, \pi_{k}(s))] > 0,
\end{align*}
which implies that there exists $s_k \in \cS$ s.t. $[Q^{\pi_k}(s_k, \pi^*(s_k)) - Q^{\pi_k}(s_k, \pi_{k}(s_k))] > 0$. 
Define policy $\pi_{k-1}$ as
\begin{align*}
\pi_{k-1}(s) = \begin{cases}
\pi^*(s) &\mbox{ if } s = s_k \\
\pi_k(s) & \mbox{ otw. }
\end{cases}
\end{align*}
We have that
\begin{align*}
 \rho^{\pi_{k-1}} &= \rho^{\pi_k}  + \sum_{s \in \cS}\mu^{\pi_{k-1}}(s) \cdot [Q^{\pi_k}(s, \pi_{k-1}(s)) - Q^{\pi_k}(s, \pi_{k}(s))]
  \\&=  \rho^{\pi_k}  + \mu^{\pi_{k-1}}(s_k) \cdot [Q^{\pi_k}(s_k, \pi^*(s_k)) - Q^{\pi_k}(s_k, \pi_{k}(s_k))] \ge \rho^{\pi_k}.
\end{align*}
Therefore, by induction, we know that there exists a policy $\pi_1$ such that $\rho^{\pi_k} \le \rho^{\pi_1}$ and $D_H(\pi_1, \pi^*) = 1$. Utilizing our initial assumption, we obtain that $\rho^{\pi^*} \ge \rho^\pi + \epsilon$ for all $\pi \ne \pi^*$, which proves that $\pi^*$ is $\epsilon$-robust optimal if the condition of the proposition holds. 
\end{proof}

\section{Proofs for Offline Attacks: Poisoning Rewards (Section \ref{sec.off.reward})}
\label{appendix.off.rewards}

We break the proof of Theorem~\ref{theorem_off_unc} into two parts: In Appendix \ref{appendix.off.lower}, we prove the lower bound in the theorem, and in Appendix \ref{appendix.off.upper_feas} we show the upper bound and feasibility claim for the attack.

\subsection{Proofs for the Lower Bound in Theorem~\ref{theorem_off_unc}}
\label{appendix.off.lower}

To prove the lower bound in Theorem~\ref{theorem_off_unc}, we 
will use a proof technique that is similar to the one presented in \cite{ma2019policy}, but adapted to our setting (the average reward criterion).  

First, let us define operator $F$ as
\begin{align}\label{contraction_operator}
F(Q, R, \rho, P, \pi)(s,a) = R(s, a) - \rho + \sum_{s' \in \cS} P(s,a,s') V^\pi(s'),
\end{align}
or in vector notation
\begin{align*}
F(Q, R, \rho, P, \pi) = R - \rho \cdot \mathbf 1 + P \cdot V^\pi,
\end{align*}
where $V^\pi(s') = Q(s',\pi(s'))$ (and $\pi$ is a deterministic policy). Furthermore, we defined the span of $X$ as $sp(X) = \max_i X(i) - \min_i X(i)$ - as argued in \cite{Puterman1994}, $sp$ is a seminorm. 

\begin{lemma}\label{lemma_contraction}
The following holds: 
\begin{align*}
    sp(F(Q_1, R, \rho, P, \pi) - F(Q_2, R, \rho, P, \pi)) \le (1-\alpha) \cdot sp(Q_1 - Q_2),
\end{align*}
where 
\begin{align*}
\alpha = \min_{s,a,s',a'} \sum_{x \in \cS}\min \{P(s,a,x), P(s',a',x) \}.
\end{align*}
\end{lemma}
\begin{proof}
We have that 
\begin{align*}
sp(F(Q_1, R, \rho, P, \pi) - F(Q_2, R, \rho, P, \pi)) = sp(P \cdot (V_{1}^\pi - V_{2}^\pi)). 
\end{align*}
Following the proof of Proposition 6.6.1 in \cite{Puterman1994}, we obtain that for $b(x,s,a,s',a') = \min \{P(s,a,x), P(s',a',x) \}$
\begin{align*}
    &\sum_{x \in \cS} P(s,a,x) \cdot  (V_{1}^\pi(x) - V_{2}^\pi(x)) -  \sum_{x \in \cS} P(s', a',x)  \cdot (V_{1}^\pi(x) - V_{2}^\pi(x))\\
    &= \sum_{x \in \cS} (P(s,a,x) - b(x,s,a,s',a')) \cdot  (V_{1}^\pi(x) - V_{2}^\pi(x)) \\
    &-  \sum_{x \in \cS} (P(s', a',x)  - b(x,s,a,s',a'))  \cdot (V_{1}^\pi(x) - V_{2}^\pi(x))\\
    &\le \sum_{x \in \cS} (P(s,a,x) - b(x,s,a,s',a')) \cdot \max_{x'} (V_{1}^\pi(x') - V_{2}^\pi(x')) \\
    &-  \sum_{x \in \cS} (P(s', a',x)  - b(x,s,a,s',a'))  \cdot \min_{x'} (V_{1}^\pi(x') - V_{2}^\pi(x'))\\
    &= (1 - \sum_{x\in \cS} b(x,s,a,s',a')) \cdot sp(V_{1}^\pi - V_{2}^\pi) \le (1-\alpha) \cdot sp(V_{1}^\pi - V_{2}^\pi)
\end{align*}
Therefore
\begin{align*}
sp(F(Q_1, R, \rho, P, \pi) - F(Q_2, R, \rho, P, \pi)) = sp(P \cdot (V_{1}^\pi - V_{2}^\pi)) \le (1 - \alpha) \cdot sp(V_{1}^\pi - V_{2}^\pi).
\end{align*}
Now, notice that for $s_\textnormal{max} = \arg \max_s [V_{1}^\pi(s) - V_{2}^\pi(s)]$ and $s_\textnormal{min} = \arg \min_s [V_{1}^\pi(s) - V_{2}^\pi(s)]$ we have that
\begin{align*}
     V_{1}^\pi(s_\textnormal{max}) - V_{2}^\pi(s_\textnormal{max}) &= Q_1(s_\textnormal{max},\pi(s_\textnormal{max})) - Q_2(s_\textnormal{max},\pi(s_\textnormal{max})) \le  \max_{s,a} [Q_1(s,a) - Q_2(s,a)] \\
    V_{1}^\pi(s_\textnormal{min}) - V_{2}^\pi(s_\textnormal{min}) &= Q_1(s_\textnormal{min},\pi(s_\textnormal{min})) - Q_2(s_\textnormal{min},\pi(s_\textnormal{min})) \ge   \min_{s, a} [Q_1(s,a) - Q_2(s,a)]
 \end{align*}
Therefore 
$sp(V_{1}^\pi - V_{2}^\pi) \le sp(Q_1 - Q_2)$,
which implies that
\begin{align*}
    sp(F(Q_1, R, \rho, P, \pi) - F(Q_2, R, \rho, P, \pi)) \le (1-\alpha) \cdot sp(Q_1 - Q_2)
\end{align*}
\end{proof}

To obtain the statement of the theorem, we will need to relate $sp(Q_1 - Q_2)$ to $\norm{R_1 - R_2}_{\infty}$. The following lemma provides a more general result, relating Q-values, reward functions and transition matrices. 

\begin{lemma}\label{lemma_qr_relation}
Let $Q_1^\pi$ and $V_1^\pi$ denote Q and V values of policy $\pi$ in MDP $M_1 = (\cS, \cA, R_1, P_1)$ and $Q_2^\pi$ denote Q values of policy $\pi$ in MDP $M_2 = (\cS, \cA, R_2, P_2)$. The following holds:
\begin{align*}
     \norm{R_1 - R_2}_{\infty} + \norm{P_1 - P_2}_{\infty} \cdot \norm{V_1^\pi}_{\infty} \ge \frac{\alpha_2}{2} \cdot sp(Q_1^\pi - Q_2^\pi).
\end{align*}
where
\begin{align*}
\alpha_2 = \min_{s,a,s',a'} \sum_{x \in \cS}\min \{P_2(s,a,x), P_2(s',a',x) \}.
\end{align*}

\end{lemma}
\begin{proof}
 Notice that $Q_1^\pi$ and $Q_2^\pi$ satisfy 
\begin{align*}
Q_1^\pi(s,a) &= F(Q_1^\pi, R_1, \rho_1^\pi, P_1, \pi)\\
Q_2^\pi(s,a) &= F(Q_2^\pi, R_2, \rho_2^\pi, P_2, \pi),
\end{align*}
where $\rho_1^\pi$ and $\rho_2^\pi$ denote the average rewards of policy $\pi$ in $M_1$ and $M_2$, respectively. We obtain
\begin{align*}
    sp(Q_1^\pi - Q_2^\pi) &= sp(F(Q_1^\pi, R_1, \rho_1^\pi, P_1, \pi) - F(Q_2^\pi, R_2, \rho_2^\pi, P_2, \pi))\\
    &= sp(F(Q_1^\pi, R_1, \rho_1^\pi, P_1, \pi) 
    -F(Q_1^\pi, R_2, \rho_2^\pi, P_1, \pi)
    \\&+F(Q_1^\pi, R_2, \rho_2^\pi, P_1, \pi)
    -F(Q_1^\pi, R_2, \rho_2^\pi, P_2, \pi)
    \\&+F(Q_1^\pi, R_2, \rho_2^\pi, P_2, \pi)
    -F(Q_2^\pi, R_2, \rho_2^\pi, P_2, \pi))\\ 
    &\le sp(F(Q_1^\pi, R_1, \rho_1^\pi, P_1, \pi) 
    -F(Q_1^\pi, R_2, \rho_2^\pi, P_1, \pi))
    \\&+sp(F(Q_1^\pi, R_2, \rho_2^\pi, P_1, \pi)
    -F(Q_1^\pi, R_2, \rho_2^\pi, P_2, \pi))
    \\&+sp(F(Q_1^\pi, R_2, \rho_2^\pi, P_2, \pi)
    -F(Q_2^\pi, R_2, \rho_2^\pi, P_2, \pi))\\
    &\le sp(R_1 - \rho_1^\pi \cdot \mathbf 1 - R_2 + \rho_2^\pi \cdot \mathbf 1)
    + sp((P_1-P_2) \cdot V_1^\pi) + (1-\alpha_2) \cdot sp(Q_1^\pi - Q_2^\pi)
\end{align*}
where the last inequality is due to Lemma \ref{lemma_contraction} (i.e., $sp(F(Q_1^\pi, R_2, \rho_2^\pi, P_2, \pi) - F(Q_2^\pi, R_2, \rho_2^\pi, P_2, \pi)) \le (1-\alpha_2) \cdot sp(Q_1^\pi - Q_2^\pi)$). 
Due to the properties of $sp$, we have
\begin{align*}
    sp(R_1 - \rho_1^\pi \cdot \mathbf 1 - R_2 + \rho_2^\pi \cdot \mathbf 1) = sp(R_1 - R_2) \le 2\cdot \norm{R_1 - R_2}_{\infty},
\end{align*}
and 
\begin{align*}
  sp((P_1-P_2) \cdot V_1^\pi) &\le 2\cdot \norm{(P_1-P_2) \cdot V_1^\pi}_{\infty} \\
  &=2 \cdot \max_{s,a}|\sum_{s'}( P_1(s,a,s') -  P_2(s,a,s')) \cdot V_1^\pi(s')| \\
  &\le2 \cdot \max_{s,a}\sum_{s'}|( P_1(s,a,s') -  P_2(s,a,s')) \cdot V_1^\pi(s')| \\
  &\le 2 \cdot \max_{s,a}\sum_{s'}|( P_1(s,a,s') -  P_2(s,a,s'))| \cdot \max_{s''} |V_1^\pi(s'')| \\
  &\le 2 \cdot \norm{P_1 - P_2}_{\infty} \cdot \norm{V_1^\pi}_{\infty},
\end{align*}
where $\norm{P_1 - P_2}_{\infty} = \max_{s,a} \sum_{s'}|P_1(s,a,s') - P_2(s,a,s')|$.
Putting this together with the upper bound on $sp(Q_1^\pi - Q_2^\pi)$, we obtain
\begin{align*}
   2 \cdot \norm{R_1 - R_2}_{\infty} + 2 \cdot \norm{P_1 - P_2}_{\infty} \cdot \norm{V_1^\pi}_{\infty}\ge \alpha_2 \cdot sp(Q_1^\pi - Q_2^\pi),
\end{align*}
which proves the claim. 
\end{proof}{}
\vspace{-2mm}
We are now ready to prove the lower bound in Theorem \ref{theorem_off_unc}.

\textbf{Proof of the Lower Bound in Theorem~\ref{theorem_off_unc}:}
\vspace{-4mm}
\begin{proof}
 From Lemma \ref{lemma_qr_relation}, it follows that a lower bound can be obtained by bounding $sp(\widehat{Q}^{\targetpi} - \overline{Q}^{\targetpi})$ from below, where $\widehat{Q}^{\targetpi}$ are Q-values of the target policy in the modified MDP, and $\overline{Q}^{\targetpi}$ are Q values of the target policy in the initial MDP. 

Let $s'$ and $a'$ be a state action pair that satisfy: $s',a' = \arg\max_{s,a} \overline{\chi}_{\epsilon}^{\targetpi}( s, a)$. W.l.o.g. we can assume that $\overline{\chi}_{\epsilon}^{\targetpi}(s', a') > 0$, since otherwise no modifications are needed to the original MDP and the lower bound trivially holds. We have 
\begin{align*}
sp(\overline{Q}^{\targetpi} - \widehat{Q}^{\targetpi}) &= sp(\widehat{Q}^{\targetpi} - \overline{Q}^{\targetpi}) = \max_{s,a} [\widehat{Q}^{\targetpi}(s, a) - \overline{Q}^{\targetpi}(s, a)]
- \min_{s,a}[\widehat{Q}^{\targetpi}(s, a) - \overline{Q}^{\targetpi}(s, a) ] \\
&\ge 
\widehat{Q}^{\targetpi}(s', \targetpi(s')) - \overline{Q}^{\targetpi}(s', \targetpi(s')) - (\widehat{Q}^{\targetpi}(s', a') - \overline{Q}^{\targetpi}(s', a')) \\
&=
(\widehat{Q}^{\targetpi}(s', \targetpi(s')) - \widehat{Q}^{\targetpi}(s', a')) +  
(\overline{Q}^{\targetpi}(s', a') - \overline{Q}^{\targetpi}(s', \targetpi(s')) \\
&\ge \frac{\epsilon}{\overline{\mu}^{\neighbor{\targetpi}{s'}{a'}}(s')} + (\overline{Q}^{\targetpi}(s', a') - \overline{Q}^{\targetpi}(s', \targetpi(s')) \\
&=\frac{\overline{\rho}^{\neighbor{\targetpi}{s'}{a'}} - \overline{\rho}^{\targetpi} + \epsilon}{\overline{\mu}^{\neighbor{\targetpi}{s'}{a'}}(s')} = \overline{\chi}_{\epsilon}^{\targetpi}(s', a'),
\end{align*}
where we used the fact that $\widehat{Q}^{\targetpi}(s', \targetpi(s')) \ge \widehat{Q}^{\targetpi}(s', a') +  \frac{\epsilon}{\overline{\mu}^{\neighbor{\targetpi}{s'}{a'}}(s')}$ (because $\targetpi$ is robustly optimal in the modified MDP and the transition kernel did not change) and Lemma \ref{lemma_qrho_relate} (Corollary \ref{gain_diff_neighbor}) to relate Q values to average rewards $\rho$. Finally, using Lemma \ref{lemma_qr_relation} (and noting that we did not change transition matrix $\overline{P}$), we obtain 
\begin{align*}
\norm{\widehat{R} - \overline{R}}_{\infty} \ge \frac{\overline{\alpha}}{2} \cdot \norm{ \overline{\chi}_{\epsilon}^{\targetpi}}_{\infty}
\end{align*}
The inequality $\norm{\widehat{R} - \overline{R}}_{p} \ge \norm{\widehat{R} - \overline{R}}_{\infty}$ implies the lower bound. 
\end{proof}

\subsection{Proofs for the Upper Bound and Feasibility in Theorem~\ref{theorem_off_unc}}
\label{appendix.off.upper_feas}

\begin{lemma}
\label{lemma.off.reward.nontarget}
$\widehat{R} = \overline{R} - \overline{\chi}_{\epsilon}^{\targetpi}$ is the solution of the following problem:
\begin{align}
	\label{prob.off.nontarget}
	\tag{P6}
	\min_{R} &\quad \norm{R - \overline R}_{p}\\
	\notag
	s.t. &\quad \sum_{s'} \overline{\mu}^{\targetpi}(s') \cdot \ R \big (s', \targetpi(s') \big ) \ge 
	\sum_{s'} \overline{\mu}^{\neighbor{\targetpi}{s}{a}}(s') \cdot R \big (s', \neighbor{\targetpi}{s}{a}(s') \big ) + \epsilon \quad \forall s, a \ne \targetpi(s), \\
	\notag
	& \quad R(s, \targetpi(s)) = \overline{R}(s, \targetpi(s)) \quad \forall s.
\end{align}
\end{lemma}{}

\begin{proof}
Since the transitions are not changed, we have
 \begin{align*}
    \sum_{s'} \widehat{R}(s', \neighbor{\targetpi}{s}{a}(s')) \cdot \widehat{\mu}^{\neighbor{\targetpi}{s}{a}}(s')
    &=\sum_{s'} \widehat{R}(s', \neighbor{\targetpi}{s}{a}(s')) \cdot  \overline{\mu}^{\neighbor{\targetpi}{s}{a}}(s')
    \\&=\sum_{s'} \overline{R}(s', \neighbor{\targetpi}{s}{a}(s'))  \cdot  \overline{\mu}^{\neighbor{\targetpi}{s}{a}}(s') - \overline{\chi}_{\epsilon}^{\targetpi}(s,a) \cdot \overline{\mu}^{\neighbor{\targetpi}{s}{a}}(s)
    \\&=\overline{\rho}^{\neighbor{\targetpi}{s}{a}} - \overline{\chi}_{\epsilon}^{\targetpi}(s,a) \cdot \overline{\mu}^{\neighbor{\targetpi}{s}{a}}(s)\\
    &\le \overline{\rho}^{\neighbor{\targetpi}{s}{a}} - (\overline{\rho}^{\neighbor{\targetpi}{s}{a}} - \overline{\rho}^{\targetpi} + \epsilon)\\
    &= \overline{\rho}^{\targetpi} - \epsilon \\&= \widehat{\rho}^{\targetpi} - \epsilon\\
    &= \sum_{s'} \overline{\mu}^{\targetpi}(s') \cdot \ R \big (s', \targetpi(s') \big ) - \epsilon
 \end{align*}
 which shows that the attack satisfies the first constraint of the optimization problem. The second constraint is also satisfied as we have $\overline{\chi}_{\epsilon}^{\targetpi}(s, \targetpi(s)) = 0$
 \end{proof}

\textbf{Proof of Feasibility and the Upper Bound in Theorem~\ref{theorem_off_unc}:}
\begin{proof}
 Consider the attack with $\widehat{R} = \overline{R} - \overline{\chi}_{\epsilon}^{\targetpi}$. As we showed in Lemma \ref{lemma.off.reward.nontarget}, it is the solution for problem (\ref{prob.off.nontarget}) which has all the constraints in (\ref{prob.off.unc}). Thus, this attack is also a solution for (\ref{prob.off.unc}) which shows the feasibility of this optimization problem. Furthermore, this attack also provides us with an upper bound on the quality of the obtained solution, i.e., for this attack:  
 \begin{align*}
     \norm{\widehat{R} - \overline{R}}_{p} = \left (\sum_{s,a} \left ( \widehat{R}(s,a) - \overline{R}(s,a) \right )^p \right )^{1/p} = \left (\sum_{s,a}  {\overline{\chi}_{\epsilon}^{\targetpi}(s,a)}^p \right )^{1/p} =  \norm{\overline{\chi}_\epsilon^{\targetpi}}_{p},
 \end{align*}
 and we know that an optimal attack achieves a lower cost (i.e., a lower value of $\norm{\widehat{R} - \overline{R}}_{p}$). 
\end{proof}
\section{Proofs for  Offline Attacks: Poisoning Dynamics (Section \ref{sec.off.dynamics})}
\label{appendix.off.dynamics}

To prove Theorem \ref{theorem_off_dyn}, we first introduce a more constrained problem which has a tractable reformulation in Appendix~\ref{appendix.off.dyn.newprob}. This problem enables us to prove our feasibility condition in Appendix~\ref{appendix.off.dyn.proof} and will also be the basis of our online dynamic attacks. Then, in Appendix~\ref{appendix.off.dyn.proof}, we prove the theorem.

\subsection{More Constrained Problem and Reformulation}
\label{appendix.off.dyn.newprob}

Consider the following problem
\begin{align}
  	\label{prob.off.dyn.nontarget}
  	\tag{P7}
  	\min_{P, \mu^{\targetpi} \mu^{\neighbor{\targetpi}{s}{a}}} &\quad \norm{P - \overline P}_{p}\\
    \notag
  	\mbox{ s.t. }& \quad \mu^{\pi}(s) = \sum_{s'} P(s', \pi(s'), s) \cdot \mu^{\pi}(s') \quad \forall s,\\
  	\notag
 	&\quad \mu^{\neighbor{\targetpi}{s}{a}}(s') = \sum_{s''}  P(s'', \neighbor{\targetpi}{s}{a}(s''), s') \cdot \mu^{\neighbor{\targetpi}{s}{a}}(s'') \quad \forall s', s,a \ne \targetpi(s),\\
  	\notag
  	&\quad \sum_{s'} \mu^{\targetpi}(s') \cdot \ \overline R \big (s', \targetpi(s') \big ) \ge 
	\sum_{s'} \mu^{\neighbor{\targetpi}{s}{a}}(s') \cdot \overline R \big (s', \neighbor{\targetpi}{s}{a}(s') \big ) + \epsilon \quad \forall s, a \ne \targetpi(s),\\
    \notag
  	&\quad P(s,a,s') \ge \delta \cdot \overline{P}(s,a,s') \quad \forall s, a, s',\\
  	\notag
  	&\quad P(s, \targetpi(s), s') =  \overline{P}(s, \targetpi(s), s') \quad \forall s, s'.
  \end{align}
 This is the exact same problem (\ref{prob.off.dyn}) with an additional condition $P(s, \targetpi(s), s') =  \overline{P}(s, \targetpi(s), s')$ for all $s, s'$. This problem is also used in the online attacks (see problem (\ref{prob.on.known.dyn})). We will show that this problem has a simple reformulation as the following:
  \begin{align}
\label{prob.off.dyn.nontarget.new}
\tag{P8}
 	\min_{P} &\quad \norm{P - \overline{P}}_{p}\\
 	\notag
 	\mbox{s.t.} &\quad  \overline{R}(s,\targetpi(s)) + \overline{B}^{\targetpi}(s) - \overline{R}(s,a) \ge \epsilon + \sum_{s'} P(s,a,s') \cdot \overline{U}^{\targetpi}(s, s') \quad \quad \forall s,a \ne \targetpi(s), s',\\
 	\notag
 	&\quad  P(s,a,s') \ge \delta \cdot \overline{P}(s,a,s') \quad \forall s, a, s',\\
 	\notag
 	& \quad  P(s,\targetpi(s),s') =  \overline{P}(s,\targetpi(s),s') \quad \forall s, s',
\end{align}
where
\begin{align*}
    \overline{U}^{\targetpi}(s, s') = \overline{V}^{\targetpi}(s') + \epsilon\cdot \overline{T}^{\targetpi}(s', s) \cdot \ind{s' \ne s}.
\end{align*}
Note that $\overline{U}^{\targetpi}$ is obtained from the initial MDP $\overline{M}$.

Before we show these problems are equivalent, let us first characterize the minimum value of $\mu^{\neighbor{\targetpi}{s}{a}}(s)$. To do so, we utilize the diameter of Markov chain defined by MDP $\overline{M}$ and policy $\targetpi$, i.e., $\overline{D}^{\targetpi} = \max_{s,s'} \overline{T}^{\targetpi}(s,s')$. Here, $\overline{T}^{\targetpi}(s,s')$ is the expected time to reach $s'$ starting from $s$ in MDP $\overline{M}$ under policy $\targetpi$. We have: 

\begin{lemma}\label{lemma_dyn_attack_stat_dist}
Assume that for transition kernel $P$, we have 
$P(s,\targetpi(s),s') =  \overline{P}(s,\targetpi(s),s')$ for all $s, s'$, and $\mu^\pi$ denotes the stationary distribution of $\pi$ under $P$. Then, for $\mu^{\neighbor{\targetpi}{s}{a}}(s)$ we have
\begin{align*}
    \mu^{\neighbor{\targetpi}{s}{a}}(s) = \frac{1}{1 + \sum_{s' \ne s} P(s, a, s') \overline{T}^{\targetpi}(s', s)} \ge \frac{1}{1 + \overline{D}^{\targetpi}},
\end{align*}
where $\overline{D}^{\targetpi} = \max_{s',s''} \overline{T}^{\targetpi}(s',s'')$ is the diameter of MDP $\overline{M}$ with policy $\targetpi$. 
\end{lemma}  
\begin{proof}
By using the fact that $\frac{1}{\mu^{\neighbor{\targetpi}{s}{a}}(s)} = T^{\neighbor{\targetpi}{s}{a}}(s, s)$ (see Theorem~1.21 in \cite{durrett1999essentials}), where $T^{\targetpi}(s',s)$ is the expected time to reach $s$ starting from $s'$ in the corresponding Markov chain, we obtain
\begin{align*}
    \frac{1}{\mu^{\neighbor{\targetpi}{s}{a}}(s)}
    &=  T^{\neighbor{\targetpi}{s}{a}}(s, s)\\
    &=  1 + \sum_{s' \ne s} P(s, a, s') \cdot T^{\neighbor{\targetpi}{s}{a}}(s', s)\\
    &=1 +  \sum_{s'\ne s} P(s, a, s') \cdot T^{\targetpi}(s', s)\\
    &=1 + \sum_{s' \ne s} P(s, a, s') \cdot \overline{T}^{\targetpi}(s', s)\\
    &\le 1 +  \max_{s'} \overline{T}^{\targetpi}(s', s)\\
    &\le 1 + \overline{D}^{\targetpi}.
\end{align*}
The last inequality obtained from the definition of diameter $\overline{D}^{\targetpi}$, i.e., $\overline{D}^{\targetpi} = \max_{s',s''} \overline{T}^{\targetpi}(s',s'')$ and we have used the fact that transitions of $\targetpi$ are not changed.
\end{proof}

\begin{lemma}
\label{lemma.off.dyn.nontarget.reform}
Problem (\ref{prob.off.dyn.nontarget.new}) is a reformulation of (\ref{prob.off.dyn.nontarget}).
\end{lemma}

\begin{proof}
Let $\rho^\pi$, $Q^\pi$, and $\mu^\pi$ denote the average reward, Q-values, and stationary distribution of $\pi$ on MDP $M = (S, A, \overline R, P)$. The last two constraints of (\ref{prob.off.dyn.nontarget}) are repeated in (\ref{prob.off.dyn.nontarget.new}). The rest of conditions in (\ref{prob.off.dyn.nontarget}) are equivalent to $\rho^{\targetpi}  \ge \rho^{\neighbor{\targetpi}{s}{a}}+ \epsilon$ for all $s, a \ne \targetpi(s)$ which we should prove it is equivalent to the remaining constraint in (\ref{prob.off.dyn.nontarget.new}).
By using the relation between average rewards $\rho$ and $Q$ values from Lemma \ref{lemma_qrho_relate} (Corollary \ref{gain_diff_neighbor}), we can write this condition as:
\begin{align*}
     Q^{\targetpi}(s,\targetpi(s)) - Q^{\targetpi}(s,a)  \ge \frac{\epsilon}{\mu^{\neighbor{\targetpi}{s}{a}}(s)}, \forall a \ne \targetpi(s).
\end{align*}
Due to Lemma \ref{lemma_dyn_attack_stat_dist}, this can be rewritten as
\begin{align*}
     Q^{\targetpi}(s,\targetpi(s)) - Q^{\targetpi}(s,a)  \ge \epsilon + \epsilon \cdot \sum_{s' \ne s} P(s, a, s') \cdot \overline{T}^{\targetpi}(s',s), \forall a \ne \targetpi(s).
\end{align*}
By using the recurrence relation for $Q$ values, the definition of $\overline{B}^{\targetpi}$, and the fact that $P(s, \targetpi(s), .) = \overline{P}(s, \targetpi(s), .)$,  we obtain that for all $s$ and $a$, it is equivalent to the following
\begin{align}
    \label{eq_lemma_dyn_attack_cond_2}
    \epsilon &\le \overline{R}(s, \targetpi(s)) + \overline{B}^{\targetpi}(s) - \overline{R}(s, a) - \sum_{s'} P(s, a, s') \cdot \bigg[\overline{V}^{\targetpi}(s') + \epsilon \cdot \overline{T}^{\targetpi}(s', s) \cdot \ind{s' \ne s} \bigg ]\\
    &= \overline{R}(s, \targetpi(s)) + \overline{B}^{\targetpi}(s) - \overline{R}(s, a) - \sum_{s'} P(s, a, s') \cdot \overline{U}^{\targetpi}(s, s').
    \end{align}
which is the condition in (\ref{prob.off.dyn.nontarget.new})
\end{proof}{}

\subsection{Feasibility Analysis}
\label{appendix.off.dyn.proof}

In this section we analyze feasibility of dynamics offline attacks as formulated in problem (\ref{prob.off.dyn}). To give a sufficient condition for feasibility of (\ref{prob.off.dyn}), it suffices to give a sufficient condition for feasibility of (\ref{prob.off.dyn.nontarget}) as it is a more constrained problem. In Lemma \ref{lemma_dyn_attack_cond}, we give a sufficient and necessary condition for feasibility of (\ref{prob.off.dyn.nontarget}). The condition stated in Theorem~\ref{theorem_off_dyn} is a simpler one showed in Corollary~\ref{cor_dyn_attack_cond} using Lemma \ref{lemma_dyn_attack_cond}.

\begin{lemma}\label{lemma_dyn_attack_cond}
The optimization problem (\ref{prob.off.dyn.nontarget}) has a solution $\widehat{P}$ if and only if for all state $s$ and actions $a \ne \targetpi(s)$ we have $C(s,a) \ge \epsilon$, where
\begin{align*}
    C(s, a) = \overline{R}(s, \targetpi(s)) + \overline{B}^{\targetpi}(s) - \overline{R}(s, a) - (1-\delta)\cdot \min_{s'} \overline{U}^{\targetpi}(s, s') - \delta \sum_{s'} \overline{P}(s, a, s') \cdot \overline{U}^{\targetpi}(s, s').
\end{align*}
\end{lemma}  
 \begin{proof}
We define $C'(s, a)$ as
\begin{align*}
    C'(s, a) = \overline{R}(s, \targetpi(s)) + \overline{B}^{\targetpi}(s) - \overline{R}(s, a) - \sum_{s'} P(s, a, s') \cdot \overline{U}^{\targetpi}(s, s').
\end{align*}{}
The constraints of  (\ref{prob.off.dyn.nontarget.new}) which is a reformulation of (\ref{prob.off.dyn.nontarget}) imply that we should have $C'(s, a) \ge \epsilon$.
Due to the constraint $P(s,a,s') \ge \delta \cdot \overline{P}(s,a,s')$, we can write
\begin{align*}
    C'(s,a) \le \overline{R}(s, \targetpi(s)) + \overline{B}^{\targetpi}(s) - \overline{R}(s, a) - (1-\delta) \cdot \min_{s'} \overline{U}^{\targetpi}(s, s') - \delta \cdot \sum_{s'} \overline{P}(s, a, s') \cdot \overline{U}^{\targetpi}(s, s') = C(s,a).
\end{align*}
Therefore, for the attack to be feasible, it has to hold that $C(s,a) \ge \epsilon$. To see that $C(s,a) \ge \epsilon$ is also a sufficient condition, let us consider transition kernel $\widehat P$ defined as
\begin{align*}
    \widehat{P}(s, a, s') = \begin{cases}
    \overline{P}(s, \targetpi(s), s') &\mbox{ if }  a = \targetpi(s)\\
    1 - \delta  + \delta \cdot \overline{P}(s, a, s') &\mbox{ if }  a \ne \targetpi(s) \text{ and } s' = \argmin_{s''} U(s, s'')\\
    \delta \cdot \overline{P}(s, a, s') &\mbox{ otherwise }
    \end{cases}
\end{align*}
For state $s$ and action $a \ne \targetpi(s)$, let $s_{\textnormal{min}} = \argmin_{s''} U(s, s'')$.
From the definition of $Q$ values and the fact that we only change transitions for non-target state-action pairs, we obtain that
\begin{align*}
   \widehat{Q}^{\targetpi}(s,\targetpi(s)) - \widehat{Q}^{\targetpi}(s,a) &= \overline{R}(s, \targetpi(s)) - \overline{\rho}^{\targetpi} + \sum_{s'} \overline{P}(s, \targetpi(s), s') \overline{V}^{\targetpi}(s')
   \\&-\overline{R}(s, a) + \overline{\rho}^{\targetpi} - \sum_{s'} \widehat{P}(s, a, s') \overline{V}^{\targetpi}(s')
   \\&=\overline{R}(s, \targetpi(s))- \overline{R}(s, a) + \overline{B}^{\targetpi}(s) - (1-\delta) \cdot \overline{V}^{\targetpi}(s_{\textnormal{min}}) - \delta \cdot \sum_{s'} \overline{P}(s,a,s') \overline{V}^{\targetpi}(s')
   \\&=\overline{R}(s, \targetpi(s))- \overline{R}(s, a) + \overline{B}^{\targetpi}(s) - (1-\delta) \cdot \overline{U}^{\targetpi}(s, s_{\textnormal{min}}) - \delta \cdot \sum_{s'} \overline{P}(s,a,s') \overline{U}^{\targetpi}(s, s')\\
   &+(1-\delta) \cdot \epsilon \cdot  \overline{T}^{\targetpi}(s_{\textnormal{min}}, s) \cdot \mathds 1_{s_{\textnormal{min}} \ne s} +
   \delta \cdot \sum_{s'} \overline{P}(s, a, s') \cdot \epsilon \cdot \overline{T}^{\targetpi}(s', s) \cdot \mathds 1_{s' \ne s}   
   \\&=C(s, a) + \sum_{s \ne s'} \widehat{P}(s, a, s') \cdot \epsilon \cdot \overline{T}^{\targetpi}(s', s)
   \\&\ge \epsilon +  \sum_{s \ne s'} \widehat{P}(s, a, s') \cdot \epsilon \cdot \overline{T}^{\targetpi}(s', s) = \frac{\epsilon}{\widehat{\mu}^{\neighbor{\targetpi}{s}{a}}},
\end{align*}
where the last equality is due to Lemma \ref{lemma_dyn_attack_stat_dist}. Using the relation between average rewards $\rho$ and $Q$ values from Lemma \ref{lemma_qrho_relate} (Corollary \ref{gain_diff_neighbor}), we obtain the claim.  
\end{proof}

Finally, from Lemma \ref{lemma_dyn_attack_cond} we derive a simpler to express sufficient condition that we use in for the main theorem of this section.  
\begin{corollary}\label{cor_dyn_attack_cond}
The optimization problem (\ref{prob.off.dyn.nontarget}) has a solution $\widehat{P}$ if for all state $s$ and actions $a \ne \targetpi(s)$ we have $\overline{\beta}^{\targetpi}_{\delta}(s,a) \ge \epsilon \cdot (1 + \overline{D}^{\targetpi})$.
\end{corollary}
\begin{proof}
Note that $C(s,a)$ from Lemma \ref{lemma_dyn_attack_cond} is bounded by: 
\begin{align*}
    C(s, a) &= \overline{R}(s, \targetpi(s)) + \overline{B}^{\targetpi}(s) - \overline{R}(s, a) - (1-\delta)\cdot \min_{s'} \overline{U}^{\targetpi}(s, s') - \epsilon_{P} \sum_{s'} \overline{P}(s, a, s') \cdot \overline{U}^{\targetpi}(s, s')\\
    &\ge \overline{R}(s, \targetpi(s)) + \overline{B}^{\targetpi}(s) - \overline{R}(s, a) - (1-\delta)\cdot (\min_{s'} \overline{V}^{\targetpi}(s') + \epsilon \cdot \overline{D}^{\targetpi})
    -\delta \cdot (\max_{s'} \overline{V}^{\targetpi}(s') + \epsilon \cdot \overline{D}^{\targetpi})\\ 
    &=\overline{\beta}^{\pi}_{\delta}(s,a) - \epsilon\overline{D}^{\targetpi},
\end{align*}
where in the last inequality we applied the definition of $\overline{\beta}^{\targetpi}_{\delta}(s,a)$, i.e.
\begin{align*}
\overline{\beta}^{\targetpi}_{\delta}(s,a) &= \overline{R}(s, \targetpi(s)) - \overline{R}(s,a) +
     \overline{B}^{\targetpi}(s) - \overline{V}_{\textnormal{min}}^{\targetpi} - \delta \cdot sp(\overline{V}^{\targetpi}).
\end{align*}

Using the condition $\overline{\beta}^{\targetpi}_{\delta}(s,a) \ge \epsilon \cdot (1 + \overline{D}^{\targetpi})$, we obtain
\begin{align*}
    C(s, a) \ge \epsilon \cdot (1 + \overline{D}^{\targetpi}) - \epsilon\overline{D}^{\targetpi} = \epsilon,
\end{align*}
which by Lemma \ref{lemma_dyn_attack_cond} implies the statement. 
\end{proof}

\subsection{Proof of Theorem~\ref{theorem_off_dyn}}
\label{appendix.off.dyn.proof}

\begin{proof}
\textbf{Feasibility and the upper bound:}
The sufficient condition in the statement of the theorem follows directly from Corollary \ref{cor_dyn_attack_cond} by noting that a solution to the optimization problem (\ref{prob.off.dyn.nontarget}) is also a solution to the optimization problem (\ref{prob.off.dyn}). To obtain the upper bound, 
let us consider an attack of the following form: 
\begin{align}
\label{attack.off.dyn.upper}
\widehat{P}(s,a,s') = (1-\lambda(s,a)) \cdot \overline{P}(s,a,s') + \lambda(s,a) \cdot P_{s}(s,a,s')
\end{align}
where
\begin{align*}
    P_s(s,a,s') = \begin{cases}
    \overline{P}(s,a,s') &\mbox{ if } a = \pi_{T}(s)\\
    1-\delta + \delta \cdot \overline{P}(s,a,s') &\mbox{ if } a \ne \pi_{T}(s) \mbox{ and } s' = \arg\min_{s''} \overline{V}^{\targetpi}(s'')\\
    \delta \cdot \overline{P}(s,a,s') &\mbox{ otw. }\\
    \end{cases}.
\end{align*}
Here, $\lambda(s,a)$ is defined for each state-action pair separately. For the above attack, note that
\begin{align*}
     \norm{\widehat{P} - \overline{P}}_{p} = \left (\sum_{s,a}\left (\sum_{s'} |\widehat{P}(s,a,s') - \overline{P}(s,a,s')|\right)^{p} \right)^{1/p} \le 2 \cdot \left ( \sum_{s,a} \lambda^{p}(s,a) \right )^{1/p}.
 \end{align*}
Hence, we want to find the minimum $\lambda(s,a)$ (across all states $s$ and $a \ne \targetpi(s)$) for which the optimization problem is feasible. 

Due to Lemma~\ref{lemma.using_neighbors}, for the attack to be feasible, we need to satisfy
\begin{align*}
    \widehat{\rho}^{\targetpi} \ge \widehat{\rho}^{\neighbor{\targetpi}{s}{a}} + \epsilon, \forall a \ne \targetpi(s).
\end{align*}
Using Lemma \ref{lemma_qrho_relate} (Corollary \ref{gain_diff_neighbor}), we can rewrite this condition in terms of $Q$ as 
\begin{align}\label{equation_condition_qopt}
     \widehat{Q}^{\targetpi}(s,a) - \widehat{Q}^{\targetpi}(s,\targetpi(s))   \le -\frac{\epsilon}{\widehat{\mu}^{\neighbor{\targetpi}{s}{a}}(s)}, \forall a \ne \targetpi(s).
\end{align}
Moreover, $\widehat{V}^{\targetpi} = \overline{V}^{\targetpi}$ for the considered attack because it does not change the transitions of the target policy. Therefore, we have
\begin{align*}
    \widehat{Q}^{\targetpi}(s,a) - \widehat{Q}^{\targetpi}(s,\targetpi(s))&=
    \overline{R}(s,a) -\widehat{\rho}^{\targetpi} + \sum_{s'}\widehat{P}(s,a,s') \cdot \overline{V}^{\targetpi}(s') \\&- \overline{R}(s,\targetpi(s)) + \widehat{\rho}^{\targetpi} -  \sum_{s'}\widehat{P}(s,\targetpi(s),s')\cdot \overline{V}^{\targetpi}(s')\\
    &=\overline{R}(s,a) + \sum_{s'}[(1-\lambda(s,a)) \cdot \overline{P}(s,a,s') + \lambda(s,a) \cdot P_s(s,a,s')]\cdot \overline{V}^{\targetpi}(s') \\&- \overline{R}(s,\targetpi(s))  -  \sum_{s'}\overline{P}(s,\targetpi(s),s') \cdot \overline{V}^{\targetpi}(s')\\
    &= (1-\lambda(s,a)) \cdot [\overline{R}(s,a) -\overline{\rho}^{\targetpi} + \sum_{s'}\overline{P}(s,a,s') \cdot \overline{V}^{\targetpi}(s') 
    \\&- \overline{R}(s,\targetpi(s)) + \overline{\rho}^{\targetpi} -  \sum_{s'}\overline{P}(s,\targetpi(s),s')\cdot \overline{V}^{\targetpi}(s')
    ]
    \\&+\lambda(s,a)\cdot
    [\overline{R}(s,a) + (1-\delta) \cdot \min_{s'} \overline{V}^{\targetpi}(s') + \delta \cdot \sum_{s'} \overline{P}(s,a ,s') \cdot \overline{V}^{\targetpi}(s')\\
    &- \overline{R}(s,\targetpi(s)) 
    -\sum_{s'} \overline{P}(s,\targetpi(s),s') \cdot \overline{V}^{\targetpi}(s')]\\
    &\le (1-\lambda(s,a)) \cdot [\overline{Q}^{\targetpi}(s,a) - \overline{Q}^{\targetpi}(s,\targetpi(s))]
    \\&+\lambda(s,a)\cdot
    [\overline{R}(s,a) - \overline{R}(s,\targetpi(s)) + \overline{V}_{\textnormal{min}}^{\targetpi}
    - \overline{B}^{\targetpi}(s) + \delta \cdot (\max_{s'} \overline{V}^{\targetpi}(s') - \overline{V}_{\textnormal{min}}^{\targetpi})]\\
    &= (1-\lambda(s,a)) \cdot \frac{ \overline{\rho}^{\neighbor{\targetpi}{s}{a}} -
    \overline{\rho}^{\targetpi}
    }{\overline{\mu}^{\neighbor{\targetpi}{s}{a}}(s)}
    -\lambda(s,a)\cdot \overline{\beta}^{\targetpi}_{\delta}(s,a) \\&= (1-\lambda(s,a)) \cdot \overline{\chi}^{\targetpi}_0(s,a)
    -\lambda(s,a)\cdot \overline{\beta}^{\targetpi}_{\delta}(s,a).
\end{align*}
By Lemma \ref{lemma_dyn_attack_stat_dist}, $\widehat{\mu}^{\neighbor{\targetpi}{s}{a}}(s) \ge \frac{1}{1+\overline{D}^{\targetpi}}$. Hence, we know that the attack is successful when $\widehat{Q}^{\targetpi}(s,a) - \widehat{Q}^{\targetpi}(s,\targetpi(s)) \le -\epsilon \cdot (1+\overline{D}^{\targetpi})$, which by the above inequality implies that a sufficient condition for the attack to be successful is
\begin{align*}
    \lambda(s,a) \cdot \overline{\chi}^{\targetpi}_0(s,a)
    +\lambda(s,a)\cdot \overline{\beta}^{\targetpi}_{\delta}(s,a) \ge \epsilon \cdot (1+\overline{D}^{\pi^T}) + \overline{\chi}^{\targetpi}_0(s,a).
\end{align*} 
Furthermore, we also know that if $\overline{\rho}^{\targetpi} \ge \overline{\rho}^{\neighbor{\targetpi}{s}{a}} + \epsilon$, we can set $\lambda(s,a) = 0$. Therefore, the attack is successful if $\lambda(s,a)$ satisfies
\begin{align*}
    \lambda(s,a) \ge \frac{\overline{\chi}^{\targetpi}_0(s,a) + \epsilon \cdot (1+\overline{D}^{\pi^T})}{ \overline{\chi}^{\targetpi}_0(s,a) + \overline{\beta}^{\targetpi}_{\delta}(s,a)} \cdot \ind{\overline{\chi}_{\epsilon}^{\targetpi}(s,a) > 0}.
\end{align*}

By choosing the lower bound for $\lambda(s,a)$, we obtain the  upper bound in the statement of the theorem, i.e.:
\begin{align}
\label{proof.off.dyn.cost_value}
\norm{\widehat{P} - \overline{P}}_{p} \le 2 \cdot \norm{\Lambda}_{p},
\end{align}{}
where $\Lambda(s,a) = \frac{\overline{\chi}^{\targetpi}_0(s,a) + \epsilon \cdot (1+\overline{D}^{\pi^T})}{ \overline{\chi}^{\targetpi}_0(s,a) + \overline{\beta}^{\targetpi}_{\delta}(s,a)} \cdot \ind{\overline{\chi}_{\epsilon}^{\targetpi}(s,a) > 0}$. 

\textbf{Lower bound:} To prove the lower bound, we follow the same arguments as in the proof of Theorem~\ref{theorem_off_unc} (the proof for the lower bound). Let $s'$ and $a'$ be a state action pair that satisfy: $s',a' = \arg\max_{s,a} \overline{\chi}^{\targetpi}( s, a)$.
Let's consider the case when $\overline{\chi}^{\targetpi}(s', a') > 0$.
We have
\begin{align*}
sp(\overline{Q}^{\targetpi} - \widehat{Q}^{\targetpi}) &= sp(\widehat{Q}^{\targetpi} - \overline{Q}^{\targetpi}) = \max_{s,a} [\widehat{Q}^{\targetpi} - \overline{Q}^{\targetpi}]
- \min_{s,a}[\widehat{Q}^{\targetpi} - \overline{Q}^{\targetpi} ] \\
&\ge 
\widehat{Q}^{\targetpi}(s', \targetpi(s')) - \overline{Q}^{\targetpi}(s', \targetpi(s')) - (\widehat{Q}^{\targetpi}(s', a') - \overline{Q}^{\targetpi}(s', a')) \\
&=
(\widehat{Q}^{\targetpi}(s', \targetpi(s')) - \widehat{Q}^{\targetpi}(s', a')) +  
(\overline{Q}^{\targetpi}(s', a') - \overline{Q}^{\targetpi}(s', \targetpi(s')) \\
&\ge \frac{\epsilon}{\widehat{\mu}^{\neighbor{\targetpi}{s'}{a'}}(s')} + (\overline{Q}^{\targetpi}(s', a') - \overline{Q}^{\targetpi}(s', \targetpi(s')) \\
&\ge \epsilon + \frac{\overline{\rho}^{\neighbor{\targetpi}{s'}{a'}} - \overline{\rho}^{\targetpi}}{\overline{\mu}^{\neighbor{\targetpi}{s'}{a'}}(s')} > \overline{\chi}^{\targetpi}_0(s', a'),
\end{align*}
where we used the fact that $\widehat{Q}^{\targetpi}(s', \targetpi(s')) \ge \widehat{Q}^{\targetpi}(s', a') +  \frac{\epsilon}{\widehat{\mu}^{\neighbor{\targetpi}{s'}{a'}}(s')}$ (because $\targetpi$ is robustly optimal in the modified MDP) and Lemma \ref{lemma_qrho_relate} (Corollary \ref{gain_diff_neighbor}) to relate Q values to average rewards $\rho$.
When $\overline{\chi}^{\targetpi}(s', a') = 0$, we know that $sp(\overline{Q}^{\targetpi} - \widehat{Q}^{\targetpi}) \ge 0$ due to the properties of $sp$. Therefore, $sp(\overline{Q}^{\targetpi} - \widehat{Q}^{\targetpi}) = sp(\widehat{Q}^{\targetpi} - \overline{Q}^{\targetpi}) \ge \norm{\overline{\chi}^{\targetpi}_0}_{\infty}$. 

By using Lemma \ref{lemma_qr_relation} (but now noting that we did not change rewards $\overline{R}$), we obtain 
\begin{align*}
\norm{\widehat{P} - \overline{P}}_{\infty} \cdot \norm{\overline{V}^{\targetpi}}_{\infty} \ge \frac{\widehat \alpha}{2} \cdot \norm{\overline{\chi}^{\targetpi}_0}_{\infty}.
\end{align*}
Factor $\widehat{\alpha}$ can be bounded by
\begin{align*}
    \widehat{\alpha} &= \min_{s,a,s',a'} \sum_{x} \min \{ \widehat{P}(s,a,x) , \widehat{P}(s',a',x) \} \\
    &\ge \min_{s,a,s',a'} \sum_{x} \min \{ \delta \cdot \overline{P}(s,a,x) , \delta \cdot  \overline{P}(s',a',x) \} \\
    &= \delta \cdot \min_{s,a,s',a'} \sum_{x} \min \{ \overline{P}(s,a,x) , \cdot  \overline{P}(s',a',x) \} \\
    &= \delta \cdot \overline{\alpha}.
\end{align*}

Putting this together with the previous expression, we obtain that
\begin{align*}
\norm{\widehat{P} - \overline{P}}_{\infty} \cdot \norm{\overline{V}^{\targetpi}}_{\infty} \ge \frac{\delta \cdot \overline{\alpha}}{2}.
\cdot \norm{\overline{\chi}^{\targetpi}_0}_{\infty}.
\end{align*}
The inequality $\norm{\widehat{P} - \overline{P}}_{p} \ge \norm{\widehat{P} - \overline{P}}_{\infty}$ implies the lower bound. 

\end{proof}

\subsection{Choosing $\delta$}

While $\epsP$ can be set to small values, making the corresponding constraint in \eqref{prob.off.dyn} a relatively weak condition, it is important to note that its value controls parameters of MDP $\widehat M$ that are important for practical considerations in the offline setting. 
Moreover, since $\epsP$ is a parameter in the optimization problems \eqref{prob.on.known.dyn} and \eqref{prob.off.dyn.nontarget} (and \eqref{prob.off.dyn.nontarget.new}), its value is also important for the online setting. 

In the case of attacks on a learning agent, our results have dependency on the agent's regret, which in turn depends on the properties of MDP $\widehat M$. For example, if the agent adopts UCRL as its learning procedure, its regret will depend on the diameter of $\widehat M$. Hence, $\epsP$ should be adjusted based on time horizon $T$, so that the parameters of MDP $\widehat M$ relevant for the agent's regret do not outweigh time horizon $T$. 

In the case of attacks on a planning agent, setting $\epsP$ to small values could result in a solution $\widehat M$ that has a large mixing time, in which case average reward $\rho$ might not approximate well the average of obtained rewards in a finite horizon (e.g., see \cite{even2005experts}).
This means that the choice of $\epsP$ should account for the finiteness of time horizon in practice. 

We leave a more detailed analysis that includes these considerations for future work.

\section{Proofs for Online Attacks: Lemma~\ref{lemma.on.known.nontarget} (Section
\ref{sec.on.ideas-definitions})}
\label{appendix.on.general}

\paragraph{Proof of Lemma~\ref{lemma.on.known.nontarget}}
Assume the learner follows an algorithm $\textsc{Alg}$ which chooses the action from a distribution based on the previous observations. Theorem 5.5.1 in \cite{Puterman1994} shows that a history-independent algorithm $\Pi$ exists that chooses the action $a_t$ form the distribution $q_{t, s_t}$ where the distributions $q_{t,s}$ are fixed and we have
\begin{gather}
\label{equal_dist}
    \forall s,a,t: \Pru{\textsc{Alg}}{s_t =s, a_t = a} =  \Pru{\Pi}{s_t =s, a_t = a},\\
    \label{equal_decision}
    q_{t,s}(a) = \Pru{\textsc{Alg}}{a_t = a | s_t = s}.
\end{gather}
Here, $\mathbb{P}_{\textsc{Alg}}$ and $\mathbb{P}_{\Pi}$ denote probabilities in the cases that the learner follows $\textsc{Alg}$ and $\Pi$, respectively.
Equation (\ref{equal_dist}) means $\Pi$ has the same expected regret and missmatches as $\textsc{Alg}$. Thus it suffices to prove the theorem for $\Pi$.

First, we extend the definitions of $P$, $R$, and $Q$ defined in Appendix \ref{appendix_background} as
\begin{gather*}
    Q(s, d) = \expctu{a\sim d}{Q(s, a)}\\
    R(s, d) = \expctu{a\sim d}{R(s, a)}\\
    P(s, d, s') = \expctu{a\sim d}{P(s, a, s')}
\end{gather*}{}
for distribution $d$ on the actions.

Now let $M = (\cS, \cA, R, P)$ be the environment. Denote the $Q$ values, $V$ values, and average reward of $\pi$ on $M$ by $Q^{\pi}$, $V^{\pi}$, and $\rho^\pi$, respectively. Also let $\rho^*$ be average reward of $\targetpi$ as $\targetpi$ is $\epsilon$-robust optimal on $M$. By Corollary \ref{gain_diff_neighbor}, for any $s, a \ne \targetpi(s)$ we have
\begin{align*}
    V^{\targetpi}(s) - Q^{\targetpi}(s, a)&= 
    \frac{1}{\mu^{\neighbor{\targetpi}{s}{a}}(s)}\big(
    \rho^* - \rho^{\neighbor{\targetpi}{s}{a}}
    \big)\\
    &\ge \frac{\epsilon}{\mu_{\textnormal{max}}}
\end{align*}{}
Defining $\eta = \epsilon/\mu_{\textnormal{max}}$, we can write
\begin{align*}
    Q^{\targetpi}(s, q_{t,s}) &= \sum_{a\ne\targetpi(s)} q_{t,s}(a)Q^{\targetpi}(s,a) + q_{t,s}(\targetpi(s))V^{\targetpi}(s)\\
    &\le \sum_{a\ne\targetpi(s)} q_{t,s}(a)\big(V^{\targetpi}(s) - \eta\big) + q_{t,s}(\targetpi(s))V^{\targetpi}(s)\\
    &= V^{\targetpi}(s) - \eta\bigg(1 - q_{t,s}(\targetpi(s))\bigg)\\
    &= V^{\targetpi}(s) - \eta e_{t,s},
\end{align*}{}
defining $e_{t,s} = 1 - q_{t,s}(\targetpi(s))$. Using the definition of $Q^{\targetpi}$ we have
\begin{align*}
    R(s, q_{t,s}) + \sum_{s'}P(s, q_{t,s}, s')V^{\targetpi}(s') - V^{\targetpi}(s) + \eta e_{t,s} \le \rho^*.
\end{align*}{}
This can be written in vector notation as
\begin{align*}
    R_{q_t} + (P_{q_t} - I)V^{\targetpi} + \eta e_t \le \rho^*\mathbf 1.
\end{align*}
Now let $d_t(s) = \Pru{\Pi}{s_t = s}$, and $d_t$ be the row vector of that. Thus, $d_{t+1} = d_{t} P_{q_t}$.
Multiplying the last inequality by $d_t$ from left gives
\begin{align*}
    &d_{t}R_{q_t} + (d_{t+1} - d_t)V^{\targetpi} + \eta  d_t e_t \le \rho^*\\
    \Rightarrow & \eta  d_t e_t \le \rho^* - \expct{r_t} + (d_t - d_{t+1})V^{\targetpi}.
\end{align*}
Summing the inequality for $t = 0$ to $T - 1$:
\begin{align*}
    \eta\sum_{t=0}^{T-1} d_t e_t &\le T\rho^* - \expct{\sum_{t=0}^{T-1} r_t} + (d_0 - d_{T})V^{\targetpi}\\
    &\le \expct{\regret(T, M)} + 2\norm{V^{\targetpi}}_\infty
\end{align*}{}

Now note that
\begin{align*}
    \sum_{t=0}^{T-1} d_t e_t &= \sum_{t=0}^{T-1} \sum_s d_t(s) e_{t,s}\\
    &= \sum_{t=0}^{T-1} \sum_s \Pru{\Pi}{s_t = s} \bigg(1 - q_{t,s}(\targetpi(s))\bigg)\\
    &= \sum_{t=0}^{T-1}\sum_s \Pru{\Pi}{s_t = s} \Pru{\Pi}{a_t \ne \targetpi(s_t) | s_t = s}\\
    &= \sum_{t=0}^{T-1} \Pru{\Pi}{a_t \ne \targetpi(s_t)}\\
    &= T\expct{\avgmissm(T)}.
\end{align*}{}
Therefore, we have 
\begin{align*}
    \expct{\avgmissm(T)} &\le \frac{\expct{\regret(T, M)} + 2\norm{V^{\targetpi}}_\infty}{T\eta}\\
    &= \frac{\mu_{\textnormal{max}}}{\epsilon \cdot T}\Big(\expct{\regret(T, M)} + 2\norm{V^{\targetpi}}_\infty\Big).
\end{align*}{}


\section{Proofs For Online Attacks: Poisoning Rewards  (Section \ref{sec.on.reward})}
\label{appendix.on.rewards}

\paragraph{Proof of Theorem \ref{theorem.on.known.reward}}
From Lemma \ref{lemma.off.reward.nontarget}, 
 $\widehat{R} = \overline{R} - \overline{\chi}^{\targetpi}_\epsilon$ is a solution for (\ref{prob.on.known.reward_n1}) which is the same as (\ref{prob.off.unc}) with an additional constraint. This means that $\widehat{R}$ is also a solution for (\ref{prob.off.unc}). Thus, $\targetpi$ is $\epsilon$-robust optimal in $\widehat{M} = (S, A, \widehat{R}, \overline{P})$. It is also obvious that the attack makes the learner learn in the MDP $\widehat{M}$. Lemma~\ref{lemma.on.known.nontarget} immediately proves the bound on expected average missmatches:
 $$
 \expct{\avgmissm(T)} \le \frac{K(T, \widehat{M})}{T}.
 $$
Since $\overline{\chi}^{\pi_T}_\epsilon(s, a) = 0$ for $a = \pi_T(s))$, we have
\begin{align*}
    \expct{\sum_{t=0}^{T-1} (\widehat{R}_t(s_t, a_t) - \overline{R}(s_t, a_t))^p} &= 
    \expct{\sum_{t=0}^{T-1} \overline{\chi}^{\targetpi}_\epsilon(s_t, a_t)^p}\\
    &= \expct{\sum_{t=0}^{T-1} \ind{a_t \ne \pi(s_t)} \overline{\chi}^{\targetpi}_\epsilon(s_t, a_t)^p}\\
    &\le \norm{ \overline{\chi}^{\targetpi}_\epsilon}_\infty^p T \expct{\avgmissm(T)}\\
    &
    \le 
        \norm{ \overline{\chi}^{\targetpi}_\epsilon}_\infty^p K(T, \widehat{M}),
\end{align*}{} 
As $p  \ge 1$, note that the function $f(x) = x^{1/p}$ is concave, so by jensen inequality, for a random variable $X$ we have $\expct{f(X)} \le f(\expct{X})$. We can write
\begin{align*}
    \expct{\avgcost(T)} &= \frac{1}{T}\expct{\left(\sum_{t=1}^T (\widehat{R}_t(s_t, a_t) - \overline{R}(s_t, a_t))^p\right)^{1/p}}\\
    &\le \frac{1}{T}\expct{\sum_{t=1}^T (\widehat{R}_t(s_t, a_t) - \overline{R}(s_t, a_t))^p}^{1/p}\\
    &\le \frac{\norm{ \overline{\chi}^{\targetpi}_\epsilon}_\infty}{T} K(T, \widehat{M})^{1/p}
\end{align*}{}

\section{Proofs For Online Attacks: Poisoning Dynamics (Section \ref{sec.on.dynamics})}
\label{appendix.on.dynamics}

As we mentioned in the main text, the optimization problem \eqref{prob.on.known.dyn} can be transformed into a tractable convex problem with liner constraints -- this is shown in Appendix \ref{appendix.off.dynamics}, where we used such an attack to obtain the upper bound of Theorem \ref{theorem_off_dyn}. 
We refer the reader to Section \ref{appendix.off.dyn.newprob}, and in particular, the optimization problems \eqref{prob.off.dyn.nontarget} and \eqref{prob.off.dyn.nontarget.new} for more details.

\paragraph{Proof of Theorem \ref{theorem.on.known.dyn}}

By Lemma~\ref{lemma.off.dyn.nontarget.reform}, the solution $\widehat{P}$ for the problem (\ref{prob.on.known.dyn}) can be efficiently calculated with a convex program with linear constraints. As (\ref{prob.on.known.dyn}) is the same as (\ref{prob.off.dyn}) with just an additional constraint, this $\widehat{P}$ is a feasible solution for (\ref{prob.off.dyn}) and by Theorem \ref{theorem_off_dyn}, $\targetpi$ is $\epsilon$-robust optimal in $\widehat{M} = (S, A, \overline{R}, \widehat{P})$ which is what the learner observes. Lemma \ref{lemma.on.known.nontarget} gives the bound for expected average missmatches:
$$
\expct{\avgmissm(T)}\le \frac{K(T, \widehat{M})}{T}
$$

To show the upper bound we consider the attack (\ref{attack.off.dyn.upper}) we used to prove the upper bound in Theorem \ref{theorem_off_dyn}:
\begin{align*}
    \widehat{P}'(s,a,s') = (1-\Lambda(s,a)) \cdot \overline{P}(s,a,s') + \Lambda(s,a) \cdot P_{s}(s,a,s')
\end{align*}{}

Since we are using the optimal solution of (\ref{prob.on.known.dyn}), and by equation (\ref{proof.off.dyn.cost_value}) we have
$$
\norm{\widehat{P} - \overline{P}}_{p} \le \norm{\widehat{P}' - \overline{P}}_{p} \le 2 \cdot \norm{\Lambda}_{p}
$$
Using this we can write
\begin{align*}
	\expct{\sum_{t=1}^T \norm{\widehat{P}_t(s_t, a_t, .) - \overline{P}(s_t, a_t, .)}_1^p}
	&=
	\expct{\sum_{t=1}^T \norm{\widehat{P}(s_t, a_t, .) - \overline{P}(s_t, a_t, .)}_1^p}\\
	&= \expct{\sum_{t=1}^T \ind{a_t \ne \targetpi(s_t)}\norm{\widehat{P}(s_t, a_t, .) - \overline{P}(s_t, a_t, .)}_1^p}\\
	&\le \norm{\widehat{P} - \overline{P}}_\infty^p T\expct{\avgmissm(T}\\
	&\le \big( 2 \cdot \norm{\Lambda}_{\infty}\big)^p K(T, \widehat{M})
\end{align*}
As $p \ge 1$ similar to proof of Theorem \ref{theorem.on.known.reward} by jensen inequality we have
\begin{align*}
    \expct{\avgcost(T)} 
    &= \frac{1}{T}\expct{\left(\sum_{t=1}^T \norm{\widehat{P}_t(s_t, a_t, .) - \overline{P}(s_t, a_t, .)}_1^p\right)^{1/p}}\\
    &\le \frac{1}{T}\expct{\sum_{t=1}^T \norm{\widehat{P}_t(s_t, a_t, .) - \overline{P}(s_t, a_t, .)}_1^p}^{1/p}\\
    & \le \frac{2}{T} \norm{\Lambda}_{\infty} K(T, \widehat{M})^{1/p}
\end{align*}

}
}
{}
\end{document}